\documentclass[final]{siamltex1213}
\usepackage{amsmath,amssymb,amsfonts,color,subfigure}
\usepackage{ifpdf}
\usepackage{hyperref}
\usepackage{multirow}
\usepackage{booktabs}
\graphicspath{{..}{figures/}}

\newtheorem{remark}{Remark}
\newtheorem{algorithm}{\textbf{\textup{Algorithm}}}

\newcommand{\scal}[2]{\left\langle #1,#2 \right\rangle}
\def\Carre#1#2{\vbox{
   \hrule height .#2pt
   \hbox{\vrule width .#2pt height #1pt \kern #1pt
      \vrule width .#2pt}
   \hrule height .#2pt}}

\def\R{\mathbb{R}}

\def\N{\mathbb{N}}

\def\dist{\textup{dist}}

\def\arg{\textup{arg}\,}

\usepackage{nicefrac}
\usepackage{xcolor}
\usepackage{xspace}
\usepackage{enumerate}
\usepackage{mathrsfs}

\renewcommand{\phi}{\varphi}

\newcommand{\res}{r}

\newcommand{\norm}[2][]{\|{#2}\|_{#1}}
\newcommand{\G}{\operatorname{Graph}}
\newcommand{\map}[3]{#1\colon #2 \to #3}

\newcommand{\enquote}[1]{``#1''}

\newcommand{\cd}[1]{\ensuremath{C^{#1}}}

\newcommand{\FF}{F}

\newcommand{\KL}{Kurdyka-{\L}ojasiewicz\xspace}
\newcommand{\dom}{\operatorname{dom}}

\newcommand{\img}{u}
\newcommand{\noisy}{u^0}

\newcommand{\cE}{\mathcal{E}}
\newcommand{\cI}{I}

\newcommand{\vt}{\vartheta}

\begin{document}

\title{iPiano: Inertial Proximal Algorithm for Non-convex Optimization}
\author{Peter~Ochs\footnotemark[2]
  \and Yunjin~Chen\footnotemark[3]
  \and Thomas~Brox\footnotemark[2]
  \and Thomas~Pock\footnotemark[3]
\thanks{
Thomas Pock acknowledges support from the Austrian science fund (FWF) under the START project \textit{BIVISION}, No. Y729. Peter Ochs and Thomas Brox acknowledge funding by the German Research Foundation (DFG grant BR 3815/5-1).
}
}
\maketitle

\renewcommand{\thefootnote}{\fnsymbol{footnote}}

\footnotetext[2]{P.~Ochs and T.~Brox are with the Department of Computer Science and with the BIOSS Centre for Biological Signalling Studies, University of Freiburg, Georges-K\"ohler-Allee 052, 79110 Freiburg, Germany.\\
  E-mail:~\texttt{\{ochs,brox\}@cs.uni-freiburg.de}}
\footnotetext[3]{Y.~Chen and T.~Pock are with the Institute for Computer Graphics and Vision, Graz University of Technology,  Inffeldgasse 16, A-8010 Graz, Austria.\\
  E-mail:~\texttt{\{cheny,pock\}@icg.tugraz.at}}

\renewcommand{\thefootnote}{\arabic{footnote}}

\begin{abstract}
In this paper we study an algorithm for solving a minimization problem composed of a differentiable (possibly non-convex) and a convex (possibly non-differentiable) function. The algorithm iPiano combines forward-backward splitting with an inertial force. It can be seen as a non-smooth split version of the Heavy-ball method from Polyak. A rigorous analysis of the algorithm for the proposed class of problems yields global convergence of the function values and the arguments. This makes the algorithm robust for usage on non-convex problems. The convergence result is obtained based on the \KL inequality. This is a very weak restriction, which was used to prove convergence for several other gradient methods. First, an abstract convergence theorem for a generic algorithm is proved, and, then iPiano is shown to satisfy the requirements of this theorem. Furthermore, a convergence rate is established for the general problem class. We demonstrate iPiano on computer vision problems: image denoising with learned priors and diffusion based image compression.
\end{abstract}

\begin{keywords}
non-convex optimization, Heavy-ball method, inertial forward-backward splitting, \KL inequality, proof of convergence
\end{keywords}

\section{Introduction}

The gradient method is certainly one of the most fundamental but also one of the most simple algorithms to solve smooth convex optimization problems. In the last decades, the gradient method has been modified in many ways. One of those improvements is to consider so-called multi-step schemes~\cite{Polyak64,Nest04}. It has been shown that such schemes significantly boost the performance of the plain gradient method. Triggered by practical problems in signal processing, image processing and machine learning, there has been an increased interest in so-called composite objective functions, where the objective function is given by the sum of a smooth function and a non-smooth function with an easy to compute proximal map. This initiated the development of the so-called proximal gradient or forward-backward method~\cite{LM79}, that combines explicit (forward) gradient steps w.r.t. the smooth part with proximal (backward) steps w.r.t. the non-smooth part.

In this paper, we combine the concepts of multi-step schemes and the proximal gradient method to efficiently solve a certain class of \emph{non-convex}, non-smooth optimization problems. Although, the transfer of knowledge from convex optimization to non-convex problems is very challenging, it aspires to find efficient algorithms for certain non-convex problems. Therefore, we consider the subclass of non-convex problems
\[
    \min_{x\in \R^N}\ f(x) + g(x)\,,
\]
where $g$ is a \emph{convex (possibly non-smooth)} and $f$ is a \emph{smooth (possibly non-convex)} function. The sum $f+g$ comprises non-smooth, non-convex functions. Despite the non-convexity, the structure of $f$ being smooth and $g$ being convex makes the forward-backward splitting algorithm well-defined. Additionally, an inertial force is incorporated into the design of our algorithm, which we termed \emph{iPiano}. Informally, the update scheme of the algorithm that will be analyzed is
\[
    x^{n+1} = (I+\alpha \partial g)^{-1}(x^n-\alpha \nabla f(x^n) + \beta(x^n-x^{n-1}))\,,
\]
where $\alpha$ and $\beta$ are the step size parameters. The term $x^n-\alpha \nabla f(x^n)$ is referred as \emph{forward step}, $\beta(x^n-x^{n-1})$ as \emph{inertial term}, and $(I+\alpha \partial g)^{-1}$ as backward or \emph{proximal step}.

For $g\equiv 0$ the proximal step is the identity and the update scheme is usually referred as \emph{Heavy-ball method}. This reduced iterative scheme is an explicit finite differences discretization of the so-called \textit{Heavy-ball with friction} dynamical system
\[
  \ddot x(t) + \gamma \dot x(t) + \nabla f(x(t)) = 0\,.
\]
It arises when Newton's law is applied to a point subject to a constant friction $\gamma>0$ (of the velocity $\dot x(t)$) and a gravity potential $f$. This explains the naming \enquote{Heavy-ball method} and the interpretation of $\beta(x^n-x^{n-1})$ as inertial force.

Setting $\beta=0$ results in the forward-backward splitting algorithm, which has the nice property that in each iteration the function value decreases. Our convergence analysis reveals that the additional inertial term prevents our algorithm from monotonically decreasing the function values. Although this may look like a limitation on first glance, demanding monotonically decreasing function values anyway is too strict as it does not allow for provably optimal schemes. We refer to a statement of Nesterov \cite{Nest04}: \enquote{In convex optimization the optimal methods never rely on relaxation. Firstly, for some problem classes this property is too expensive. Secondly, the schemes and efficiency estimates of optimal methods are derived from some global topological properties of convex functions}\footnote{Relaxation is to be interpreted as the property of monotonically decreasing function values in this context. Topological properties should be associated with geometrical properties.}. The negative side of better efficiency estimates of an algorithm is usually the convergence analysis. This is even true for convex functions. In case of non-convex and non-smooth functions, this problem becomes even more severe.

\paragraph{Contributions}
Despite this problem, we can establish convergence of the sequence of function values for the general case, where the objective function is only required to be a composition of a convex and a differentiable function. Regarding the sequence of \emph{arguments} generated by the algorithm, existence of a converging subsequence is shown. Furthermore, we show that each limit point is a critical point of the objective function.

To establish convergence of the \emph{whole} sequence in the non-convex case is very hard. However, with slightly more assumptions to the objective, namely that it satisfies the \KL inequality \cite{Loj63,Loj93,Kurd98}, several algorithms have been shown to converge \cite{CPR13,ABS13,AB08,ABPS10}. In \cite{ABS13} an abstract convergence theorem for descent methods with certain properties is proved. It applies to many algorithms. However, it can not be used for our algorithm. Based on their analysis, we prove an abstract convergence theorem for a different class of descent methods, which applies to iPiano. By verifying the requirements of this abstract convergence theorem, we manage to also show such a strong convergence result. From a practical point of view of image processing, computer vision, or machine learning, the \KL inequality is almost always satisfied. For more details about properties of \KL functions and a taxonomy of functions that have this property, we refer to \cite{ABS13,BDLM10,Kurd98}.

The last part of the paper is devoted to experiments. We exemplarily present results on computer vision tasks, such as denoising and image compression, and show that entering the staggering world of non-convex functions pays off in practice.

\section{Related Work}

\paragraph{Forward-backward splitting}

In convex optimization, splitting algorithms usually originate from the proximal point algorithm \cite{Rock76}. It is a very general algorithm, and results on its convergence affect many other algorithms. Practically, however, computing one iteration of the algorithm can be as hard as the original problem. Among the strategies to tackle this problem are splitting approaches like Douglas-Rachford \cite{LM79,EB92}, several primal-dual algorithms \cite{CP11,PC11,HY12}, and forward-backward splitting \cite{LM79,CW05,BT09b,Nest04}; see \cite{CP11a} for a survey.

Especially the forward-backward splitting schemes seem to be appealing to generalize to non-convex problems. This is due to their simplicity and the existence of simpler formulations in some special cases like, for example, the gradient projection method, where the backward-step is the projection onto a set \cite{LP66,Gold64}. In \cite{FM81} the classical forward-backward algorithm, where the backward step is the solution of a proximal term involving a convex function, is studied for a non-convex problem. In fact, the same class of objective functions as in the present paper is analyzed. The algorithm presented here comprises the algorithm from \cite{FM81} as a special case. Also Nesterov \cite{Nest13} briefly accounts this algorithm in a general setting. Even the reverse setting is generalized in the non-convex setting \cite{ABS13,BL09}, namely where the backward-step is performed on a non-smooth non-convex function.

As the amount of data to be processed is growing and algorithms are supposed to exploit all the data in each iteration, inexact methods become interesting, though we do not consider erroneous estimates in this paper. Forward-backward splitting schemes also seem to work for non-convex problems with erroneous estimates \cite{Sra12,Sol97}. A mathematical analysis of inexact methods can be found, e.g., in \cite{CPR13,ABS13}, but with the restriction that the method is explicitly required to decrease the function values in each iteration. The restriction comes with significantly improved results with regard of the convergence of the algorithm. The algorithm proposed in this paper provides strong convergence results, although it does not require the function values to decrease.

\paragraph{Optimization with inertial forces}

In his seminal work~\cite{Polyak64}, Polyak investigates multi-step schemes to accelerate the gradient method. It turns out that a particularly interesting case is given by a two-step algorithm, which has been coined the \textit{Heavy-ball} method. The name of the method is because it can be interpreted as an explicit finite differences discretization of the so-called \textit{Heavy-ball with friction} dynamical system. It differs from the usual gradient method by adding an inertial term that is computed by the difference of the two preceding iterations. Polyak showed that this method can speed up convergence in comparison to the standard gradient method, while the cost of each iteration stays basically unchanged.

The popular accelerated gradient method of Nesterov~\cite{Nest04} obviously shares some similarities with the Heavy-ball method, but it differs from it in one regard: while the Heavy-ball method uses gradients based on the current iterate, Nesterov's accelerated gradient method evaluates the gradient at points that are extrapolated by the inertial force. On strongly convex functions, both methods are equally fast (up to constants), but Nesterov's accelerated gradient method converges much faster on weakly convex functions~\cite{Drori2013}.

The Heavy-ball method requires knowledge about the function parameters (Lipschitz constant of the gradient and the modulus of strong convexity) to achieve the optimal convergence rate, which can be seen as a disadvantage. Interestingly, the conjugate gradient method for minimizing strictly convex quadratic problems can be expressed as Heavy-ball method. Hence, it can be seen as a special case of the Heavy-ball method for quadratic problems. In this special case, no additional knowledge is required about the function parameters, as the algorithm parameters are computed online.

The Heavy-ball method was originally proposed for minimizing differentiable convex functions, but it has been generalized in different ways. In~\cite{ZK93}, it has been generalized to the case of smooth non-convex functions. It is shown that, by considering an appropriate Lyapunov objective function, the iterations are attracted by the connected components of stationary points. In Section~\ref{sec:alg} it will become evident that the non-convex Heavy-ball method is a special case of our algorithm, and also the convergence analysis of \cite{ZK93} shows some similarities to ours.

In~\cite{AlvarezAttouch2001,Alvarez2003}, the Heavy-ball method has been extended to maximal monotone operators, e.g., the subdifferential of a convex function. In a subsequent work~\cite{MoudafiOliny}, it has been applied to a forward-backward splitting algorithm, again in the general framework of maximal monotone operators.

\section{An abstract convergence result} \label{sec:abstract-convergence}
\subsection{Preliminaries}
We consider the Euclidean vector space $\R^N$ of dimension $N\geq 1$ and denote the standard inner product by $\scal{\cdot}{\cdot}$ and the induced norm by $\norm[2]{\cdot}^2 := \sqrt{\scal{\cdot}{\cdot}}$. Let $\map \FF{\R^N}{\R\cup\{+\infty\}}$ be a proper lower semi-continuous function.
\begin{definition}[effective domain, proper]
  The \emph{(effective) domain} of $\FF$ is defined by $\dom \FF := \{ x\in \R^N: \FF(x)<+\infty\}$. The function is called \emph{proper}, if $\dom \FF$ is nonempty.
\end{definition}

In order to give a sound description of the first order optimality condition for a non-convex non-smooth optimization problem, we have to introduce the generalization of the subdifferential for convex functions.

\begin{definition} [Limiting-subdifferential] \label{def:differential}
  The \emph{limiting-subdifferential} (or simply \emph{subdifferential}) is defined by (see~\cite[Def. 8.3]{Rock98})
  \begin{equation}\label{eq:subdifferential}
    \partial \FF (x) = \{ \xi \in \R^N \vert\, \exists y_k \to x,\, \FF(y_k)     \to \FF(x),\, \xi_k \to \xi,\, \xi_k\in \widehat\partial \FF(y_k) \} \,,
  \end{equation}
  which makes use of the \emph{Fr\'echet subdifferential} defined by
  \[
  \widehat \partial \FF(x) = \{ \xi \in \R^N \vert\,
  \liminf_{\substack{y\to x\\ y\neq x}} \tfrac 1{\norm[2]{x-y}} \left(
    \FF(y) - \FF(x) - \scal{y-x}{\xi}\right) \geq 0  \}\,,
  \]
  when $x\in \dom \FF$ and by $\widehat \partial \FF(x)=\varnothing$ else.
\end{definition}

The \emph{domain} of the subdifferential is $\dom \partial F := \{ x\in \R^N \vert\, \partial F(x) \neq \varnothing\}$.\\

In what follows, we will consider the problem of finding a critical point $x^*\in \dom \FF$ of $\FF$, which is characterized by the necessary first-order optimality condition $0\in \partial \FF(x^*)$.\\ 

We state the definition of the \KL property from \cite{ABPS10}. 
\begin{definition}[\KL property]\label{def:KL-property}\
\begin{enumerate}
\item The function $\map{\FF}{\R^N}{\R\cup\{\infty\}}$ has the \KL property at $x^*\in\dom\partial\FF$, if there exist $\eta\in(0,\infty]$, a neighborhood $U$ of $x^*$ and a continuous concave function $\map{\phi}{[0,\eta)}{\R_+}$ such that $\phi(0) = 0$, $\phi \in \cd 1((0,\eta))$, for all $s\in (0,\eta)$ it is $\phi^\prime(s) > 0$, and for all $x\in U\cap [\FF(x^*) < \FF < \FF(x^*)+\eta]$ the \KL inequality holds, i.e., 
\[
  \phi^\prime(\FF(x)-\FF(x^*))\dist(0,\partial\FF(x)) \geq 1 \,.
\]
\item If the function $\FF$ satisfies the \KL inequality at each point of $\dom \partial \FF$, it is called KL function.
\end{enumerate}
\end{definition}

Roughly speaking, this condition says that we can bound the subgradient of a function from below by a reparametrization of its function values. In the smooth case, we can also say that up to a reparametrization the function $h$ is sharp, meaning that any non-zero gradient can be bounded away from $0$. This is sometimes called a desingularization. It has been shown in \cite{ABPS10} that a proper lower semi-continuous extended valued function $h$ always satisfies this inequality at each non-stationary point. For more details and other interpretations of this property, also for different formulations, we refer to \cite{BDLM10}. 

A big class of functions that have the KL-property is given by real semi-algebraic functions \cite{ABPS10}. Real semi-algebraic functions are defined as functions whose graph is a real semi-algebraic set. 
\begin{definition}[real semi-algebraic set]\label{def:semi-alg-set}
  A subset $S$ of $\R^N$ is \emph{semi-algebraic}, if there exists a finite number of real polynomials $\map{P_{i,j},Q_{i,j}}{\R^N}{\R}$ such that 
  \[
    S = \bigcup_{j=1}^p \bigcap_{i=1}^q \{x\in\R^N :\, P_{i,j}(x) = 0\text{ and } Q_{i,j}<0 \}\,.
  \]
\end{definition}

\subsection{Inexact descent convergence result for KL functions} \label{subsec:abstract-convergence}

In the following, we prove an abstract convergence result for a sequence $(z^n)_{n\in\N}:=(x^n,x^{n-1})_{n\in\N}$ in $\R^{2N}$, $x^n\in\R^N$, $x^{-1}\in\R^N$, satisfying certain basic conditions, $\N:=\{0,1,2,\ldots\}$. For convenience we use the abbreviation $\Delta_n := \norm[2]{x^n - x^{n-1}}$ for $n\in \N$. We fix two positive constants $a>0$ and $b>0$ and consider a proper lower semi-continuous function $\map{\FF}{\R^{2N}}{\R\cup\{\infty\}}$. Then, the conditions we require for $(z^n)_{n\in\N}$ are
\begin{enumerate}
\item[(H1)] For each $n\in \N$, it holds
\[
    \FF(z^{n+1}) + a \Delta_n^2 \leq \FF(z^n)\,.
\]
\item[(H2)] For each $n\in\N$, there exists $w^{n+1}\in\partial\FF(z^{n+1})$ such that
\[
    \norm[2]{w^{n+1}} \leq \frac b2 (\Delta_{n} + \Delta_{n+1})\,.
\]
\item[(H3)] There exists a subsequence $(z^{n_j})_{j\in\N}$ such that
\[
    z^{n_j}\to \tilde z \quad \text{and}\quad \FF(z^{n_j}) \to \FF(\tilde z)\,,\qquad \text{as } j\to\infty\,.
\]
\end{enumerate}

Based on these conditions, we derive the same convergence result as in \cite{ABS13}. The statements and proofs of the subsequent results follow the same ideas as \cite{ABS13}. We modified the involved calculations according to our conditions H1, H2, and H3. 

\begin{remark}
These conditions are very similar to the ones in \cite{ABS13}, however, they are not identical. The difference comes from the fact that \cite{ABS13} does not consider a two-step algorithm.
\begin{itemize}
  \item In \cite{ABS13} the corresponding condition to H1 (sufficient decrease condition) is $\FF(x^{n+1}) + a \Delta_{n+1}^2 \leq \FF(x^n)$.
  \item The corresponding condition to H2 (relative error condition) is $\norm[2]{w^{n+1}} \leq b \Delta_{n+1}$. In some sense, our condition H2 accepts a larger relative error. 
  \item H3 (continuity condition) in \cite{ABS13} is the same here, but for $(x^{n_j})_{j\in\N}$.
\end{itemize}
\end{remark}

\begin{remark}
  Our proof and the proof in \cite{ABS13} mainly differ in the calculations that are involved, the outline is the same. There is hope to find an even more general convergence result, which comprises ours and \cite{ABS13}. 
\end{remark}

\begin{lemma}\label{lem:main-theorem-convergence}
Let $\map{\FF}{\R^{2N}}{\R\cup\{\infty\}}$ be a proper lower semi-continuous function which satisfies the \KL property at some point $z^*=(x^*,x^*)\in \R^{2N}$. Denote by $U$, $\eta$ and $\map{\phi}{[0,\eta)}{\R_+}$ the objects appearing in Definition~\ref{def:KL-property} of the KL property at $z^*$. Let $\sigma, \rho >0$ be such that $B(z^*,\sigma)\subset U$ with $\rho \in (0,\sigma)$, where $B(z^*,\sigma):=\{z\in\R^{2N} : \norm[2]{z-z^*} < \sigma\}$.

Furthermore, let $(z^n)_{n\in\N}=(x^n,x^{n-1})_{n\in\N}$ be a sequence satisfying Conditions~H1, H2, and 
\begin{equation}\label{eq:mt-eqA}
  \forall n\in\N:\quad z^n\in B(z^*,\rho) \Rightarrow z^{n+1}\in B(z^*,\sigma)\text{ with } \FF(z^{n+1}),\FF(z^{n+2}) \geq \FF(z^*) \,.
\end{equation}
Moreover, the initial point $z^0=(x^0,x^{-1})$ is such that $\FF(z^*) \leq \FF(z^0) < \FF(z^*) + \eta$ and 
\begin{equation}\label{eq:mt-eqD}
\norm[2]{x^*-x^0} + \sqrt{\frac {\FF(z^0) - \FF(z^*)}a} + \frac ba \phi(\FF(z^0) - \FF(z^*)) < \frac \rho 2 \,.
\end{equation}
Then, the sequence $(z^n)_{n\in\N}$ satisfies
\begin{equation}
  \forall n\in\N: z^n\in B(z^*,\rho),\quad \sum_{n=0}^\infty \Delta_n < \infty,\quad \FF(z^n) \to \FF(z^*), \text{ as } n\to\infty\,,
\end{equation}
$(z^n)_{n\in\N}$ converges to a point $\bar z=(\bar x,\bar x)\in B(z^*,\sigma)$ such that $\FF(\bar z) \leq \FF(z^*)$.
If, additionally, Condition~H3 is satisfied, then $0\in \partial \FF(\bar z)$ and $\FF(\bar z) = \FF(z^*)$.
\end{lemma}
\begin{proof}
The key points of the proof are the facts that for all $j\geq 1$: 
\begin{gather}
  z^j\in B(z^*,\rho)\qquad \text{and} \label{eq:keyA}\\
  \sum_{i=1}^j \Delta_{i} \leq \frac 12(\Delta_0-\Delta_j) + \frac ba [\phi(\FF(z^1)-\FF(z^*)) - \phi(\FF(z^{j+1})-\FF(z^*)))]  \label{eq:keyB}
\end{gather}
Let us first see that $\phi(\FF(z^{j+1})-\FF(z^*))$ is well-defined. By Condition~H1, $(\FF(z^n))_{n\in\N}$ is non-increasing, which shows $\FF(z^{n+1}) \leq \FF(z^0) < \FF(z^*) + \eta$. Combining this with \eqref{eq:mt-eqA} implies $\FF(z^{n+1})-\FF(z^*)\geq 0$. 

As for $n\geq 1$ the set $\partial \FF(z^n)$ is nonempty (see Condition~H2) every $z^n$ belongs to $\dom \FF$. For notational convenience, we define
\[
  D^\phi_n := \phi(\FF(z^n)-\FF(z^*)) - \phi(\FF(z^{n+1})-\FF(z^*))\,.
\]
Now, we want to show that for $n\geq 1$ holds: if $\FF(z^n)<\FF(z^*)+\eta$ and $z^n\in B(z^*,\rho)$, then
\begin{equation}\label{eq:mt-eqB}
  2\Delta_n \leq \tfrac ba D^\phi_n + \tfrac 12 (\Delta_n + \Delta_{n-1}) \,.
\end{equation}
Obviously, we can assume that $\Delta_n\neq 0$ (otherwise it is trivial), and therefore H1 and \eqref{eq:mt-eqA} imply $\FF(z^n)>\FF(z^{n+1})\geq \FF(z^*)$. The KL inequality shows $w^{n} \neq 0$ and H2 shows $\Delta_n + \Delta_{n-1}>0$. Since $w^n\in \partial \FF(z^n)$, using KL inequality and H2, we obtain
\[
  \phi^\prime(\FF(z^n) - \FF(z^*)) \geq \frac 1{\norm[2]{w^n}} \geq \frac 2{b(\Delta_{n-1} + \Delta_n)}\,.
\]
As $\phi$ is concave and increasing ($\phi^\prime >0$), Condition~H1 and \eqref{eq:mt-eqA} yield
\[
  D^\phi_n \geq \phi^\prime(\FF(z^n) - \FF(z^*))(\FF(z^n) - \FF(z^{n+1})) \geq \phi^\prime(\FF(z^n) - \FF(z^*)) a\Delta_n^2 \,.
\]
Combining both inequalities results in 
\[
   (\tfrac ba D^\phi_n)\tfrac 12 (\Delta_{n-1}+\Delta_n) \geq \Delta_n^2\,,
\]
which by applying $2\sqrt{uv} \leq u+v$ establishes \eqref{eq:mt-eqB}. 

As \eqref{eq:mt-eqA} does only imply $z^{n+1}\in B(z^*,\sigma)$, $\sigma > \rho$, we can not use \eqref{eq:mt-eqB} directly for the whole sequence. However, \eqref{eq:keyA} and \eqref{eq:keyB} can be shown by induction on $j$. For $j=0$, \eqref{eq:mt-eqA} yields $z^1\in B(z^*,\sigma)$ and $\FF(z^1),\FF(z^2) \geq \FF(z^*)$. From Condition~H1 with $n=1$, $\FF(z^2) \geq \FF(z^*)$ and $\FF(z^1)\leq \FF(z^0)$, we infer
\begin{equation} \label{eq:mt-eqC}
  \Delta_1 \leq \sqrt{\frac {\FF(z^1) - \FF(z^2)}a} \leq \sqrt{\frac {\FF(z^0) - \FF(z^*)}a}\,,
\end{equation}
which combined with \eqref{eq:mt-eqD} leads to 
\[
 \norm[2]{x^*-x^1} \leq \norm[2]{x^0-x^*} + \Delta_1 \leq \norm[2]{x^0-x^*} + \sqrt{\frac {\FF(z^0) - \FF(z^*)}a} < \frac \rho 2\,,
\]
and therefore $z^1\in B(z^*, \rho)$. Direct use of \eqref{eq:mt-eqB} with $n=1$ shows that \eqref{eq:keyB} holds with $j=1$.

Suppose \eqref{eq:keyA} and \eqref{eq:keyB} are satisfied for $j\geq 1$. Then, using the triangle inequality and \eqref{eq:keyB}, we have
\[
  \begin{array}{rcl}
    \norm[2]{z^*-z^{j+1}} 
    & \leq & \norm[2]{x^*-x^{j+1}} + \norm[2]{x^*-x^j}  \\
    & \leq & 2\norm[2]{x^*-x^0} + 2\sum_{i=1}^{j} \Delta_{i} + \Delta_{j+1}\\
    & \leq & 2\norm[2]{x^*-x^0} + (\Delta_0- \Delta_{j}) + \Delta_{j+1} \\
    &     & \ 2\frac ba [\phi(\FF(z^1)-\FF(z^*)) - \phi(\FF(z^{j+1})-\FF(z^*)))] \\
    & \leq & 2\norm[2]{x^*-x^0} + \Delta_0 + \Delta_{j+1} + 2\frac ba [\phi(\FF(z^0)-\FF(z^*))] \,,
  \end{array}
\]
which shows, using $\Delta_{j+1}\leq \sqrt{\frac 1a (\FF(z^{j+1}) - \FF(z^{j+2}))} \leq \sqrt{\frac 1a (\FF(z^0) - \FF(z^*))}$ and \eqref{eq:mt-eqD}, that $z^{j+1}\in B(z^*,\rho)$. As a consequence \eqref{eq:mt-eqB}, with $n=j+1$, can be added to \eqref{eq:keyB} and we can conclude \eqref{eq:keyB} with $j+1$. This shows the desired induction on $j$.

Now, the finiteness of the length of the sequence $(x^n)_{n\in\N}$, i.e., $\sum_{i=1}^\infty \Delta_i < \infty$, is a consequence of the following estimation, which is implied by \eqref{eq:keyB},
\[
  \sum_{i=1}^{j} \Delta_{i} \leq \tfrac 12\Delta_0 + \tfrac ba \phi(\FF(z^1)-\FF(z^*)) < \infty\,.
\]
Therefore, $x^n$ converges to some $\bar x$ as $n\to\infty$, and $z^n$ converges to $\bar z=(\bar x,\bar x)$. As $\phi$ is concave, $\phi^\prime$ is decreasing. Using this and Condition~H2 yields $w^n\to 0$ and $\FF(z^n)\to\zeta\geq \FF(z^*)$. Suppose we have $\zeta > \FF(z^*)$, then KL-inequality reads $\phi^\prime(\zeta-\FF(z^*))\norm[2]{w^n} \geq 1$ for all $n\geq 1$, which contradicts $w^n\to 0$. 

Note that, in general, $\bar z$ is not a critical point of $\FF$, because the limiting subdifferential requires $\FF(z^n)\to \FF(\bar z)$ as $n\to\infty$. When the sequence $(z^n)_{n\in\N}$ additionally satisfies Condition~H3, then $\tilde z = \bar z$, and $\bar z$ is a critical point of $\FF$, because $\FF(\bar z) = \lim_{n\to\infty} \FF(z^n) = \FF(z^*)$.
\qquad\end{proof}
\bigskip

\begin{remark} 
  The only difference to \cite{ABS13} with respect to the assumptions is \eqref{eq:mt-eqA}. In~\cite{ABS13}, $z^n\in B(z^*,\rho)$ implies $\FF(z^{n+1})\geq \FF(z^*)$, whereas we require $\FF(z^{n+1})\geq\FF(z^*)$ and $\FF(z^{n+2})\geq\FF(z^*)$. However, as Theorem~\ref{thm:conv-abstract} shows, this does not weaken the convergence result compared to \cite{ABS13}. In fact, Corollary~\ref{cor:seq-ball-condition}, which assumes $\FF(z^n) \geq \FF(z^*)$ for all $n\in\N$ and which is also used in \cite{ABS13}, is key in Theorem~\ref{thm:conv-abstract}.
\end{remark}

The next corollary and the subsequent theorem follow as in \cite{ABS13} by replacing the calculation with our conditions.

\begin{corollary}\label{cor:seq-ball-condition}
  Lemma~\ref{lem:main-theorem-convergence} holds true, if we replace \eqref{eq:mt-eqA} by 
  \[
    \eta < a (\sigma - \rho)^2 \qquad\text{and}\qquad \FF(z^n) \geq \FF(z^*),\ \forall n\in\N\,.
  \]
\end{corollary}
\begin{proof}
  By Condition~H1, for $z^n\in B(z^*,\rho)$, we have 
  \[
    \Delta_{n+1}^2 \leq \frac{\FF(z^{n+1}) - \FF(z^{n+2})}a \leq \frac \eta a < (\sigma - \rho)^2 \,.
  \]
  Using the triangle inequality on $\norm[2]{z^{n+1} - z^*}$ shows that $z^{n+1}\in B(z^*,\sigma)$, which implies \eqref{eq:mt-eqA} and concludes the proof.
\qquad\end{proof}
\bigskip

The work that is done in Lemma~\ref{lem:main-theorem-convergence} and Corollary~\ref{cor:seq-ball-condition} allows us to formulate an abstract convergence theorem for sequences satisfying the Conditions~H1, H2, and H3. It follows, with a few modifications, as in \cite{ABS13}.

\begin{theorem}[Convergence to a critical point]\label{thm:conv-abstract} 
Let $\map{\FF}{\R^{2N}}{\R\cup\{\infty\}}$  be a proper lower semi-continuous function and $(z^n)_{n\in\N}=(x^n,x^{n-1})_{n\in\N}$ a sequence that satisfies H1, H2, and H3. Moreover, let $\FF$ have the \KL property at the cluster point $\tilde x$ specified in H3. 

Then, the sequence $(x^n)_{n=0}^\infty$ has finite length, i.e., $\sum_{n=1}^\infty \Delta_{n}<\infty$, and converges to $\bar x=\tilde x$ as $n\to\infty$, where $(\bar x,\bar x)$ is a critical point of $\FF$. 
\end{theorem}
\begin{proof}
  By Condition~H3, we have $z^{n_j}\to \bar z=\tilde z$ and $\FF(z^{n_j})\to \FF(\bar z)$ for a subsequence $(z^{n_j})_{n\in\N}$. This, together with the non-decreasingness of $(\FF(z^n))_{n\in\N}$ (by Condition~H1), imply that $\FF(z^n) \to \FF(\bar z)$ and $\FF(z^n)\geq \FF(\bar z)$ for all $n\in\N$. The KL-property around $\bar z$ states the existence of quantities $\phi$, $U$, and $\eta$ as in Definition~\ref{def:KL-property}. Let $\sigma>0$ be such that $B(\bar z,\sigma) \subset U$ and $\rho\in(0,\sigma)$. Shrink $\eta$ such that $\eta < a(\sigma-\rho)^2$ (if necessary). As $\phi$ is continuous, there exists $n_0\in\N$ such that $\FF(z^n)\in [\FF(\bar z),\FF(\bar z)+\eta )$ for all $n\geq n_0$ and 
\[
  \norm[2]{x^*-x^{n_0}} + \sqrt{\frac {\FF(z^{n_0}) - \FF(z^*)}a} + \frac ba \phi(\FF(z^{n_0}) - \FF(z^*)) < \frac \rho 2\,.
\]
Then, the sequence $(y^n)_{n\in \N}$ defined by $y^n=z^{n_0+n}$ satsifies the conditions in Corollary~\ref{cor:seq-ball-condition}, which concludes the proof.
\qquad\end{proof}
\bigskip

\section{The proposed algorithm - iPiano} \label{sec:alg}

\subsection{The optimization problem}
We consider a structured non-smooth non-convex optimization problem with a proper lower semi-continuous extended valued function $\map{h}{\R^N}{\R\cup\{+\infty\}}$, $N\geq 1$:
\begin{equation}\label{eq:problem-class}
\min_{x \in \R^N} \; h(x) = \min_{x \in \R^N} \; f(x) + g(x)\,,
\end{equation}
which is composed of a $\cd 1$-smooth (possibly non-convex) function $\map f{\R^N}{\R}$ with $L$-Lipschitz continuous gradient on $\dom g$, $L>0$, and a convex (possibly non-smooth) function $\map g {\R^N}{\R\cup\{+\infty\}}$. Furthermore, we require $h$ to be coercive, i.e., $\norm[2]{x} \rightarrow +\infty$ implies $h(x) \to +\infty$, and bounded from below by some value $\underline h > -\infty$.\\

The proposed algorithm, which is stated in Subsection~\ref{subsec:generic-alg}, seeks for a critical point $x^*\in \dom h$ of $h$, which is characterized by the necessary first-order optimality condition $0\in \partial h(x^*)$. In our case, this is equivalent to
\[
-\nabla f(x^*) \in \partial g(x^*)\,.
\]
This equivalence is explicitly verified in the next subsection, where we collect some details and state some basic properties, which are used in the convergence analysis in Subsection~\ref{subsec:nc-conv-ana}.

\subsection{Preliminaries}

Consider the function $f$ first. It is required to be $\cd 1$-smooth with $L$-Lipschitz continuous gradient on $\dom g$, i.e., there exists a constant $L > 0$ such that
\begin{equation}\label{eq:lipschitz-assumption}
\norm[2]{\nabla f(x) - \nabla f(y)} \leq L \norm[2]{x-y}\,,\quad \forall x,y \in \dom g\,.
\end{equation}
This directly implies that $\dom h = \dom g$ is a non-empty convex set, as $\dom g\subset \dom f$. This property of $f$ plays a crucial role in our convergence analysis due to the following lemma (stated as in \cite{ABS13}).
\begin{lemma}[descent lemma]\label{upper-bound-f}
  Let $\map{f}{\R^N}{\R}$ be a $\cd 1$-function with $L$-Lipschitz continuous gradient $\nabla f$ on $\dom g$. Then for any $x,y \in \dom g$ it holds that
  \begin{equation}\label{eq:upper-bound-f}
  f(x) \leq f(y) + \scal{\nabla f(y)}{x-y} + \frac{L}{2}\norm[2]{x-y}^2\,.
  \end{equation}
\end{lemma}
\begin{proof}
See for example \cite{Nest04}.
\qquad\end{proof}
\bigskip

We assume that the function $g$ is a proper lower semi-continuous convex function with an efficient to compute proximal map.
\begin{definition}[proximal map]
Let $g$ be a proper lower semi-continuous convex function. Then, we define the \emph{proximal map}
\[
(I + \alpha \partial g)^{-1}(\hat x) := \arg\min_{x\in \R^N} \frac{\norm[2]{x-\hat x}^2}{2} + \alpha g(x)\,,
\]
where $\alpha >0$ is a given parameter, $I$ is the identity map, and $\hat x \in \R^N$.
\end{definition}

An important (basic) property that the convex function $g$ contributes to the convergence analysis is the following:
\begin{lemma}\label{lower-bound-g}
  Let $g$ be a proper lower semi-continuous convex function, then it holds for any $x,y \in \dom g$, $s \in \partial g(x)$ that
  \begin{equation}\label{eq:lower-bound-g}
  g(y) \geq g(x) + \scal{s}{y-x}\,.
  \end{equation}
\end{lemma}
\begin{proof}
  This result follows directly from the convexity of $g$.
\qquad\end{proof}
\bigskip

Finally, consider the optimality condition $0\in \partial h(x^*)$ more in detail. The following proposition proves the equivalence to $-\nabla f(x^*)\in\partial g(x^*)$. The proof is mainly based on Definition~\ref{def:differential} of the limiting-subdifferential.
\begin{proposition}
  Let $h$, $f$, and $g$ be like before, i.e., let $h=f+g$ with $f$ continuously differentiable and $g$ convex. Sometimes, $h$ is then called a $\cd 1$-perturbation of a convex function. Then, for $x\in\dom h$ holds
  \[
    \partial h(x) = \nabla f(x) + \partial g(x)\,.
  \]
\end{proposition}
\begin{proof}
  We first prove \enquote{$\subset$}. Let $\xi^h \in \partial h(x)$, i.e., there is a sequence $(y_k)_{k=0}^\infty$ such that $y_k\to x$, $h(y_k)\to h(x)$, and $\xi_k^h\to \xi^h$, where $\xi_k^h\in\widehat\partial h(y_k)$. We want to show that $\xi^g:= \xi^h - \nabla f(x)\in \partial g(x)$.
  As $f\in\cd 1$ and $\xi^h \in \partial h(x)$, we have
  \begin{gather*}
     y_k \overset{k\to\infty}{\longrightarrow} x \\
     g(y_k) =  h(y_k) - f(y_k) \overset{k\to\infty}{\longrightarrow} h(x) - f(x) = g(x)\\
     \xi^g_k := \xi^h_k - \nabla f(y_k) \overset{k\to\infty}{\longrightarrow} \xi^h - \nabla f(x) =: \xi^g\,.
  \end{gather*}
  It remains to show that $\xi^g_k\in \widehat \partial g(y_k)$. First, remember that $\liminf$ is superadditive, i.e., for two sequences $(a_n)_{n=0}^\infty$, $(b_n)_{n=0}^\infty$ in $\R$ it is $\liminf_{n\to\infty} (a_n+b_n) \geq \liminf_{n\to\infty} a_n +\liminf_{n\to\infty} b_n$. However, convergence of $a_n$ implies $\liminf_{n\to\infty} (a_n+b_n) = \lim_{n\to\infty} a_n +\liminf_{n\to\infty} b_n$. This fact and again thanks to $f\in\cd 1$, we conclude
  \[
    \begin{array}{rcl}
    0 & \leq & \liminf \left( h(y_k^\prime)-h(y_k)-\scal{y_k^\prime-y_k}{\xi^h_k}\right) / \norm[2]{y_k^\prime-y_k} \\
    & \leq & \liminf \left( f(y_k^\prime)-f(y_k)+g(y_k^\prime)-g(y_k)-\scal{y_k^\prime-y_k}{\nabla f(y_k) + \xi^g_k}\right) / \norm[2]{y_k^\prime-y_k} \\
    & = & \lim \left( f(y_k^\prime)-f(y_k)-\scal{y_k^\prime-y_k}{\nabla f(y_k)}\right) / \norm[2]{y_k^\prime-y_k} \\
    & & +\ \liminf \left( g(y_k^\prime)-g(y_k)-\scal{y_k^\prime-y_k}{ \xi^g_k}\right) / \norm[2]{y_k^\prime-y_k} \\
    & = & \liminf \left( g(y_k^\prime)-g(y_k)-\scal{y_k^\prime-y_k}{ \xi^g_k}\right) / \norm[2]{y_k^\prime-y_k}\,,
    \end{array}
  \]
  where $\liminf$ and $\lim$ are over $y_k^\prime\to y_k, y_k^\prime\neq y_k$. Therefore, $\xi^g_k\in \widehat \partial g(y_k)$. \\
  The other inclusion \enquote{$\supset$} is trivial.
\qquad\end{proof}
\bigskip

As a consequence, a critical point can also be characterized by the following definition.
\begin{definition}[proximal residual]\label{def:prox-res}
  Let $f$ and $g$ be as afore. Then, we define the   \emph{proximal residual}
  \[
    \res(x) := x - (I + \partial g)^{-1}(x - \nabla f(x))\,.
  \]
\end{definition}

It can be easily seen that $\res(x) = 0$ is equivalent to $x = (I + \partial g)^{-1}(x - \nabla f(x))$ and $(I + \partial g)(x) = (I -\nabla f)(x)$, which is the first-order optimality condition. The proximal residual is defined with respect to a fixed step size of $1$. The rationale behind this becomes obvious when $g$ is the indicator function of a convex set. In this case, a small residual could be caused by small step sizes as the reprojection onto the convex set is independent of the step size.

\subsection{The generic algorithm}\label{subsec:generic-alg}
In this paper, we propose an algorithm, iPiano, with the generic formulation in Algorithm~\ref{alg:ipiano-intro}. It is a forward-backward splitting algorithm incorporating an inertial force. In the forward step, $\alpha_n$ determines the step size in the direction of the gradient of the differentiable function $f$. The step in gradient direction is aggregated with the inertial force from the previous iteration weighted by $\beta_n$. Then, the backward step is the solution of the proximity operator for the function $g$ with the weight $\alpha_n$.
\begin{figure}[h!]
\centering
\fbox{
\begin{minipage}{0.95\textwidth}
\begin{algorithm}\label{alg:ipiano-intro}
\
\textbf{i}nertial \textbf{p}rox\textbf{i}mal \textbf{a}lgorithm
for \textbf{n}on-convex \textbf{o}ptimization (iPiano)
\begin{itemize}
\item Initialization: Choose a starting point $x^0 \in \dom h$ and set $x^{-1}= x^0$. Moreover, define sequences of step size parameter $(\alpha_n)_{n=0}^\infty$ and $(\beta_n)_{n=0}^\infty$.
\item Iterations $(n\ge 0)$: Update
  \begin{equation}\label{eq:ipiano-intro-up}
    x^{n+1} = (I+\alpha_n \partial g)^{-1}(x^n-\alpha_n \nabla f(x^n)
                                          + \beta_n(x^n-x^{n-1}))\,.
  \end{equation}
  \end{itemize}
\end{algorithm}
\end{minipage}
}
\end{figure}

In order to make the algorithm specific and convergent, the step size parameters must be chosen appropriately. What \enquote{appropriately} means, will be specified in Subsection~\ref{subsec:strategies} and proved in Subsection~\ref{subsec:nc-conv-ana}.

\subsection{Rules for choosing the step size}\label{subsec:strategies}

In this subsection, we propose several strategies for choosing the step sizes. This will make it easier to implement the algorithm. One may choose among the following variants of step size rules depending on the knowledge about the objective function.

\paragraph{Constant step size scheme} The most simple one, which requires most knowledge about the objective function, is outlined in Algorithm~\ref{alg:ipiano-const-step}. All step size parameters are chosen a priori and are constant.
\begin{figure}[h!]
\centering
\fbox{
\begin{minipage}{0.95\textwidth}
\begin{algorithm}\label{alg:ipiano-const-step}
\
\textbf{i}nertial \textbf{p}rox\textbf{i}mal \textbf{a}lgorithm for
\textbf{n}on-convex \textbf{o}ptimization with constant parameter (ciPiano)
\begin{itemize}
\item Initialization: Choose $\beta \in [0,1)$, set $\alpha <
  2(1-\beta)/L$, where $L$ is the Lipschitz constant of $\nabla f$,
  choose , $x^0 \in \dom h$ and set $x^{-1}= x^0$.
  \item Iterations $(n\ge 0)$: Update $x^n$ as follows:
    \begin{equation}\label{eq:ipiano-const}
      x^{n+1} = (I+\alpha \partial g)^{-1}(x^n-\alpha \nabla f(x^n) +
      \beta(x^n-x^{n-1}))
    \end{equation}
  \end{itemize}
\end{algorithm}
\end{minipage}
}
\end{figure}
\begin{remark}
  Observe that our law on $\alpha,\beta$ is equivalent to the law found in \cite{ZK93} for minimizing a smooth non-convex function. Hence, our result can be seen as an extension of their work to the presence of an additional non-smooth convex function.
\end{remark}

\paragraph{Backtracking} The case where we have only limited knowledge about the objective function occurs more frequently. It can be very challenging to estimate the Lipschitz constant of $\nabla f$ beforehand. Using backtracking the Lipschitz constant can be estimated automatically. A sufficient condition that the Lipschitz constant at iteration $n$ to $n+1$ must satisfy is
\begin{equation}\label{eq:BT-up-cond-lip}
  f(x^{n+1}) \leq f(x^n) + \scal{\nabla f(x^n)}{x^{n+1}-x^n} + \frac{L_{n}}{2}\norm[2]{x^{n+1} - x^n}^2\,.
\end{equation}
Although, there are different strategies to determine $L_n$, the most common one is by defining an increment variable $\eta>1$ and looking for $L_n \in \{L_{n-1}, \eta L_{n-1}, \eta^2L_{n-1},\ldots\}$ minimal satisfying \eqref{eq:BT-up-cond-lip}. Sometimes, it is also feasible to decrease the estimated Lipschitz constant after a few iterations. A possible strategy is as follows: if $L_n = L_{n-1}$, then search for the minimal $L_n \in \{\eta^{-1} L_{n-1}, \eta^{-2}L_{n-1},\ldots\}$ satisfying \eqref{eq:BT-up-cond-lip}.

In Algorithm~\ref{alg:ipiano-BT} we propose an algorithm with variable step sizes. Any strategy for estimating the Lipschitz constant may be used. When changing the Lipschitz constant from one iteration to another, all step size parameters must be adapted. The rules for adapting the step sizes will be justified during the convergence analysis in Subsection~\ref{subsec:nc-conv-ana}.
\begin{figure}[h!]
\centering
\fbox{
\begin{minipage}{0.95\textwidth}
\begin{algorithm}\label{alg:ipiano-BT}
\
\textbf{i}nertial \textbf{p}rox\textbf{i}mal \textbf{a}lgorithm for
\textbf{n}on-convex \textbf{o}ptimization with backtracking (biPiano)
\begin{itemize}
\item Initialization: Choose $\delta\geq c_2 > 0$ with $c_2$ close to $0$ (e.g. $c_2:=10^{-6}$), and
  $x^0 \in \dom h$ and set $x^{-1}= x^0$.
  \item Iterations $(n\ge 0)$: Update $x^n$ as follows:
    \begin{equation}\label{eq:ipiano-BT}
      x^{n+1} = (I+\alpha_n \partial g)^{-1}(x^n-\alpha_n \nabla f(x^n) +
      \beta_n(x^n-x^{n-1}))\,,
    \end{equation}
    where $L_n >0$ satisfies \eqref{eq:BT-up-cond-lip} and
    \begin{gather*}
      \beta_n = (b-1)/(b-\frac 12)\,, \qquad  b:=(\delta+\frac{L_n}2)/(c_2+\frac{L_n}2)\,, \\
      \alpha_n=2(1-\beta_n)/(2c_2 + L_n)\,.
    \end{gather*}
  \end{itemize}
\end{algorithm}
\end{minipage}
}
\end{figure}

\paragraph{Lazy backtracking} Algorithm~\ref{alg:ipiano-BTL} presents another alternative of Algorithm~\ref{alg:ipiano-intro}. It is related to Algorithm~\ref{alg:ipiano-const-step} and~\ref{alg:ipiano-BT} in the following way. Algorithm~\ref{alg:ipiano-BTL} makes use of the Lipschitz continuity of $\nabla f$ in the sense that the Lipschitz constant is always finite. As a consequence, using backtracking with only increasing Lipschitz constants, after a finite number of iterations $n_0\in\N$ the estimated Lipschitz constant will not change anymore, and starting from this iteration the constant step size rules as in Algorithm~\ref{alg:ipiano-const-step} are applied. Using this strategies, the results that will be proved in the convergence analysis are satisfied only as soon as the Lipschitz constant is high enough and does not change anymore.

\begin{figure}[h!]
\centering
\fbox{
\begin{minipage}{0.95\textwidth}
\begin{algorithm}\label{alg:ipiano-BTL}
\
non-monotone \textbf{i}nertial \textbf{p}rox\textbf{i}mal \textbf{a}lgorithm
for \textbf{n}on-convex \textbf{o}ptimization with backtracking (nmiPiano)
\begin{itemize}
\item Initialization: Choose $\beta \in [0,1)$, $L_{-1}>0$, $\eta > 1$, and
  $x^0 \in \dom h$ and set $x^{-1}= x^0$.
  \item Iterations $(n\ge 0)$: Update $x^n$ as follows:
    \begin{equation}\label{eq:ipiano-BTL}
      x^{n+1} = (I+\alpha_n \partial g)^{-1}(x^n-\alpha_n \nabla f(x^n) +
      \beta(x^n-x^{n-1}))\,,
    \end{equation}
    where $L_n \in \{L_{n-1},\eta L_{n-1}, \eta^2L_{n-1},\ldots\}$ is minimal
    satisfying
    \begin{equation}\label{eq:BTL-up-cond-lip}
      f(x^{n+1}) \leq f(x^n) + \scal{\nabla f(x^n)}{x^{n+1}-x^n}
      + \frac{L_{n}}{2}\norm[2]{x^{n+1} - x^n}^2
    \end{equation}
    and $\alpha_n<2(1-\beta)/L_n$.
  \end{itemize}
\end{algorithm}
\end{minipage}
}
\end{figure}

\paragraph{General rule of choosing the step sizes}

Algorithm~\ref{alg:ipiano-general} defines the general rules that the step size parameters must satisfy.
\begin{figure}[h!]
\centering
\fbox{
\begin{minipage}{0.95\textwidth}
\begin{algorithm}\label{alg:ipiano-general}
\
\textbf{i}nertial \textbf{p}rox\textbf{i}mal \textbf{a}lgorithm
for \textbf{n}on-convex \textbf{o}ptimization (iPiano)
\begin{itemize}
\item Initialization: Choose $c_1,c_2>0$ close to $0$, $x^0 \in \dom h$ and set $x^{-1}= x^0$.
\item Iterations $(n\ge 0)$: Update
  \begin{equation}\label{eq:ipiano-up}
    x^{n+1} = (I+\alpha_n \partial g)^{-1}(x^n-\alpha_n \nabla f(x^n)
                                          + \beta_n(x^n-x^{n-1}))\,,
  \end{equation}
  where $L_n>0$ is the local Lipschitz constant satisfying
  \begin{equation}\label{eq:up-cond-lip}
    f(x^{n+1}) \leq f(x^n) + \scal{\nabla f(x^n)}{x^{n+1}-x^n}
    + \frac{L_{n}}{2}\norm[2]{x^{n+1} - x^n}^2\,,
  \end{equation}
  and $\alpha_n \geq c_1$, $\beta_n\geq 0$ are chosen such that $\delta_n\geq
  \gamma_n \geq c_2$ defined by
  \begin{equation}\label{eq:def-delta-gamma}
    \delta_n := \frac 1 {\alpha_n} -\frac {L_n}2 - \frac{\beta_n}{2\alpha_n}
    \quad\text{and}\quad\gamma_n:=\frac 1 {\alpha_n} -\frac {L_n} 2
    - \frac{\beta_n}{\alpha_n}\,.
  \end{equation}
  and $(\delta_n)_{n=0}^\infty$ is monotonically decreasing.
  \end{itemize}
\end{algorithm}
\end{minipage}
}
\end{figure}
It contains the Algorithms~\ref{alg:ipiano-const-step},~\ref{alg:ipiano-BT},~and~\ref{alg:ipiano-BTL} as special instances. This is easily verified for Algorithms~\ref{alg:ipiano-const-step} and~\ref{alg:ipiano-BTL}. For Algorithm~\ref{alg:ipiano-BT} the step size rules are derived from the proof of Lemma~\ref{lem:exist-rel-parameter}.

As Algorithm~\ref{alg:ipiano-general} is the most general one, now, let us analyze the behavior of this algorithm.

\subsection{Convergence analysis} \label{subsec:nc-conv-ana}

In all what follows, let $(x^n)_{n=0}^\infty$ be the sequence generated by Algorithm~\ref{alg:ipiano-general} and with parameters satisfying the algorithm's requirements. Furthermore, for a more convenient notation we abbreviate $H_\delta(x,y) := h(x) + \delta\norm[2]{x-y}^2$, $\delta\in\R$, and $\Delta_n:=\norm[2]{x^n-x^{n-1}}$. Note, that for $x=y$ it is $H_\delta(x,y)=h(x)$.\\

Let us first verify that the algorithm makes sense. We have to show that the requirements to the parameters are not contradictory, i.e., that it is possible to choose a feasible set of parameters. In the following Lemma, we will only show existence of such a parameter set, however, the proof helps us to formulate specific step size rules.
\begin{lemma}\label{lem:exist-rel-parameter}
  For all $n\geq 0$, there are $\delta_n\geq\gamma_n$, $\beta_n\in [0,1)$, and $\alpha_n < {2(1-\beta_n)}/{L_n}$. Furthermore, given $L_n>0$, there exists a choice of parameter $\alpha_n$ and $\beta_n$ such that additionally $(\delta_n)_{n=0}^\infty$ is monotonically
  decreasing.
\end{lemma}
\begin{proof} \
  By the algorithm's requirements it is
  \[
    \delta_n = \frac 1{\alpha_n}-\frac {L_n} 2-\frac{\beta_n}{2\alpha_n}
    \geq \frac 1{\alpha_n}-\frac {L_n} 2-\frac{\beta_n}{\alpha_n}
    = \gamma_n > 0\,.
  \]
  The upper bound for $\beta_n$ and $\alpha_n$ come from rearranging
  $\gamma_n\geq c_2$ to $\beta_n \leq 1-\alpha_n L_n/2-c_2\alpha_n$ and
  $\alpha_n\leq 2(1-\beta_n)/(L_n+2c_2)$, respectively.\\
  The last statement follows by incorporating the descent property of $\delta_n$.
  Let $\delta_{-1}\geq c_2$ be chosen initially. Then, the decent property of
  $(\delta_n)_{n=0}^\infty$ requires one of the equivalent statements
  \[
    \delta_{n-1}\geq \delta_n
    \quad \Leftrightarrow \quad
    \delta_{n-1} \geq \frac 1{\alpha_n}-\frac {L_n} 2-\frac{\beta_n}{2\alpha_n}
    \quad \Leftrightarrow \quad
    \alpha_{n} \geq \frac{1-\frac {\beta_n}{2}}{\delta_{n-1}+\frac{L_n}2}
  \]
  to be true. An upper bound on $\alpha_n$ is obtained by
  \[
    \gamma_n \geq c_2
    \quad \Leftrightarrow \quad
    \alpha_n \leq \frac{1-\beta_n}{c_2+\frac{L_n}2}\,.
  \]
  The only thing that remains to show is that there exists $\alpha_n>c_1$ and
  $\beta_n\in[0,1)$ such that these two relations are fulfilled. Consider the
  condition for a non-negative gap between the upper and lower bound for
  $\alpha_n$
  \[
    \frac{1-\beta_n}{c_2+\frac{L_n}2} - \frac{1-\frac {\beta_n}{2}}{\delta_{n-1}+\frac{L_n}2} \geq 0
    \quad \Leftrightarrow \quad
    \frac{\delta_{n-1}+\frac{L_n}2}{c_2+\frac{L_n}2} \geq  \frac{1-\frac {\beta_n}{2}}{1-\beta_n}\,.
  \]
  Defining $b:=(\delta_{n-1}+\frac{L_n}2)/(c_2+\frac{L_n}2)\geq 1$, it is
  easily verified that there exists $\beta_n\in[0,1)$ satisfying the equivalent
  condition
  \begin{equation}\label{eq:step-req-beta}
    \frac{b-1}{b-\frac 12} \geq  \beta_n \,.
  \end{equation}
  As a consequence, the existence of a feasible $\alpha_n$ follows, and
  the decent property for $\delta_n$ holds.
\qquad\end{proof}
\bigskip

In the following proposition, we state a result which will be very useful. Although, iPiano does not imply a descent property of the function values, we construct a majorizing function that enjoys a monotonically descent property. This function reveals the connection to the Lyapunov direct method for convergence analysis as used in \cite{ZK93}.

\begin{proposition}\label{prop:convergence} \
  \begin{enumerate}[ (a)]
  \item\label{prop:conv-hd}
    The sequence $(H_{\delta_n}(x^n,x^{n-1}))_{n=0}^\infty$ is monotonically decreasing and thus converging. In particular, it holds
    \begin{equation}\label{eq:hd-descent}
      H_{\delta_{n+1}}(x^{n+1},x^n) \leq H_{\delta_n} (x^{n},x^{n-1})
        -\gamma_n \Delta_n^2\,.
    \end{equation}
  \item\label{prop:conv-diff-xk}
    It holds $\sum_{n=0}^\infty \Delta_n^2 < \infty$ and, thus, $\lim_{n\to\infty}\Delta_n=0$.
  \end{enumerate}
\end{proposition}
\begin{proof} \
  \begin{enumerate}[ (a)]
  \item From~\eqref{eq:ipiano-up} it follows that
  \[
  \frac{x^n-x^{n+1}}{\alpha_{n}} - \nabla f(x^n) + \frac{\beta_{n}}{\alpha_{n}}
                                      (x^n-x^{n-1}) \in \partial g(x^{n+1})
  \]
  Now using $x = x^{n+1}$ and $y=x^n$ in~\eqref{eq:upper-bound-f} and~\eqref{eq:lower-bound-g} and summing both inequalities it follows that
  \begin{eqnarray*}
    h(x^{n+1}) &\leq& h(x^n) - (\frac 1 \alpha_{n} - \frac {L_n}2)\,
                      \Delta_{n+1}^2 + \frac {\beta_{n}}{\alpha_{n}}
                      \scal{x^{n+1}-x^n}{x^n-x^{n-1}}\\
               &\leq& h(x^n) - (\frac 1 {\alpha_{n}} - \frac {L_n} 2 -
                      \frac{\beta_{n}}{2\alpha_{n}} )\, \Delta_{n+1}^2
                      + \frac{\beta_{n}}{2\alpha_{n}}\Delta_n^2\,,
  \end{eqnarray*}
  where the second line follows from $2\scal{a}{b}\leq \norm[2]{a}^2+\norm[2]{b}^2$ for vectors $a,b\in \R^N$. Then, a simple rearrangement of the terms shows
  \begin{eqnarray*}
    h(x^{n+1}) + \delta_n \Delta_{n+1}^2 \leq h(x^{n}) + \delta_n \Delta_n^2 - \gamma_n \Delta_n^2\,,
  \end{eqnarray*}
  which establishes~\eqref{eq:hd-descent} as $\delta_n$ is monotonically decreasing. Obviously, the sequence $(H_{\delta_n}(x^n,x^{n-1}))_{n=0}^\infty$ is monotonically decreasing if and only if $\gamma_n \geq 0$, which is true by the algorithm's requirements. By assumption, $h$ is bounded from below by some constant $\underline h > -\infty$, hence $(H_{\delta_n}(x^n,x^{n-1}))_{n=0}^\infty$ converges.
  \item Summing up \eqref{eq:hd-descent} from $n=0,\ldots,N$ yields (note that $H_{\delta_n}(x^0,x^{-1})=h(x^0)$)
  \begin{eqnarray*}
  \sum_{n=0}^N \gamma_n \Delta_n^2
  & \leq & \sum_{n=0}^N H_{\delta_{n}}(x^n,x^{n-1})
                      -H_{\delta_{n+1}}(x^{n+1},x^n) \\
  & = & h(x^0) - H_{\delta_{N+1}}(x^{N+1},x^N) \leq
        h(x^0) - \underline h\; < \infty \,.
  \end{eqnarray*}
  Letting $N$ tend to $\infty$ and remembering that $\gamma_N\geq c_2> 0$ holds implies the statement.
  \end{enumerate}
\qquad\end{proof}
\bigskip

\begin{remark}
  The function $H_\delta$ is a Lyapunov function for the dynamical system of described by the Heavy-ball method. It corresponds to a discretized version of the kinetic energy of the Heavy-ball with friction.
\end{remark}

In the following theorem, we state our general convergence results about Algorithm~\ref{alg:ipiano-general}.
\begin{theorem}\label{thm:convergence} \
  \begin{enumerate}[ (a)]
  \item\label{thm:conv-h-conv} The sequence $(h(x^n))_{n=0}^\infty$ converges.
  \item There exists a converging subsequence $(x^{n_k})_{k=0}^\infty$.
  \item\label{thm:limit-is-critical} Any limit point ${x^*:= \lim_{k\to\infty}x^{n_k}}$ is a critical point of~\eqref{eq:problem-class} and $h(x^{n_k}) \to h(x^*)$ as $k\to\infty$.
  \end{enumerate}
\end{theorem}
\begin{proof} \
  \begin{enumerate}[ (a)]
  \item This follows from the Squeeze theorem as for all $n\geq 0$ holds
  \[
  H_{-\delta_n}(x^n,x^{n-1}) \leq h(x^n) \leq H_{\delta_n}(x^n,x^{n-1})
  \]
  and thanks to Proposition~\ref{prop:convergence}(\ref{prop:conv-hd}) and~(\ref{prop:conv-diff-xk}) holds
  \[
    \lim_{n\to\infty}\! H_{-\delta_n}(x^n,x^{n-1})
    = \lim_{n\to\infty}\! H_{\delta_n}(x^n,x^{n-1}) -2\delta_n \Delta_n^2 \\
    = \lim_{n\to\infty}\! H_{\delta_n}(x^n,x^{n-1})\,.
  \]
  \item By Proposition~\ref{prop:convergence}(\ref{prop:conv-hd}) and $H_{\delta_0}(x^0,x^{-1})=h(x^0)$ it is clear that the whole sequence $(x^n)_{n=0}^\infty$ is contained in the level set $\{x\in \R^N:\;\underline h \leq h(x) \leq h(x^0) \}$, which is bounded thanks to the coercivity of $h$ and $\underline h = \inf_{x\in \R^N} h(x) > -\infty$. Using the Bolzano-Weierstrass theorem, we deduce the existence of a converging subsequence $(x^{n_k})_{k=0}^\infty$.
  \item To show that each limit point $x^*:=\lim_{j\to\infty} x^{n_j}$ is a critical point of~\eqref{eq:problem-class} recall that the subdifferential~\eqref{eq:subdifferential} is closed~\cite{Rock98}. Define
  \[
    \xi^j :=\frac{x^{n_j}-x^{n_j+1}}{\alpha_{n_j}} - \nabla f(x^{n_j})
    + \frac{\beta_{n_j}}{\alpha_{n_j}}(x^{n_j}-x^{n_j-1}) + \nabla f(x^{n_j+1}).
  \]
  Then, the sequence $(x^{n_j},\xi^j)\in \G(\partial h):=\{(x,\xi)\in \R^N\times \R^N\vert\, \xi \in \partial h(x)\}$. Furthermore, it holds $x^* = \lim_{j\to\infty}x^{n_j}$ and due to Proposition~\ref{prop:convergence}(\ref{prop:conv-diff-xk}), the Lipschitz continuity of $\nabla f$, and
  \[
    \norm[2]{\xi^j-0} \leq \frac 1{\alpha_{n_j}} \Delta_{n_j+1}
    + \frac{\beta_{n_j}}{\alpha_{n_j}} \Delta_{n_j}
    + \norm[2]{\nabla f(x^{n_j+1})-\nabla f(x^{n_j})}
  \]
  it holds $\lim_{j\to\infty}\xi^j=0$. It remains to show that $\lim_{j\to\infty}h(x^{n_j})=h(x^*)$. By the closure property of the subdifferential $\partial h$ it is $(x^*,0)\in \G(\partial h)$, which means that $x^*$ is a critical point of $h$.

  The continuity statement about the limiting process as $j\to\infty$ follows by the lower semi-continuity of $g$, the existence $\lim_{j\to\infty}\xi^j=0$, and the convexity property in Lemma~\ref{lower-bound-g}
  \[
    \limsup_{j\to\infty}g(x^{n_j}) =
    \limsup_{j\to\infty}g(x^{n_j}) + \scal{\xi^j}{x^*-x^{n_j}}
    \leq g(x^*) \leq \liminf_{j\to\infty} g(x^{n_j})\,.
  \]
  The first equality holds because the subadditivity of $\limsup$ becomes an equality when the limit exists for one of the two summed sequences\footnote{In general, the existence of $(\xi^j)_{j=0}^\infty$ is not guaranteed. Compared to the general case, additionally $\lim_{j\to\infty}\xi^j=0$ is known here.}, here it exists $\lim_{j\to\infty} \scal{\xi^j}{x^*-x^{n_j}} =0$. Moreover, as $f$ is differentiable it is also continuous, thus $\lim_{j\to\infty} f(x^{n_j}) = f(x^*)$. This implies $\lim_{j\to\infty}h(x^{n_j})=h(x^*)$.
  \end{enumerate}
\qquad\end{proof}
\bigskip

\begin{remark}
  The convergence properties shown in Theorem~\ref{thm:convergence} should be the basic requirement of any algorithm. Very loosely speaking, it states that the algorithm ends up in a meaningful solution. It allows us to formulate stopping conditions, e.g., the residual between successive function values.
\end{remark}

Now, using Theorem~\ref{thm:conv-abstract}, we can verify the convergence of the sequence $(x^n)_{n\in\N}$ generated by Algorithm~\ref{alg:ipiano-general}. We assume that after a finite number of steps the sequence $(\delta_n)_{n\in\N}$ is constant and consider the sequence $(x^n)_{n\in\N}$ starting from this iteration (again denoted by $(x^n)_{n\in\N}$). For example, if $\delta_n$ is determined relative to the Lipschitz constant, then as the Lipschitz constant can be assumed constant after a finite number of iterations, $\delta_n$ is also constant starting from this iteration.

\begin{theorem}[Convergence of iPiano to a critical point]\label{thm:convergence-whole-seq}
Let $(x^n)_{n\in\N}$ be generated by Algorithm~\ref{alg:ipiano-general}, and let $\delta_n=\delta$ for all $n\in\N$. Then, the sequence $(x^{n+1}, x^{n})_{n\in\N}$ satisfies H1, H2, and H3 for the function $\map{H_\delta}{\R^{2N}}{\R\cup\{\infty\}}$, $(x,y)\mapsto h(x) + \delta \norm[2]{x-y}^2$. 

Moreover, if $H_\delta(x,y)$ has the \KL property at a cluster point $(x^*,x^*)$, then the sequence $(x^n)_{n\in\N}$ has finite length, $x^n\to x^*$ as $n\to\infty$, and $(x^*,x^*)$ is a critical point of $H_\delta$, hence $x^*$ is a critical point of $h$.
\end{theorem}
\begin{proof}
First, we verify that the Assumptions H1, H2, and H3 are satisfied. We consider the sequence $z^{n} = (x^{n}, x^{n-1})$ for all $n\in\N$ and the proper lower semi-continuous function $F=H_\delta$. 
\begin{itemize}
\item Condition H1 is proved in Proposition~\ref{prop:convergence}(a) with $a=c_2\leq \gamma_n$.
\item To proof Condition H2, consider $w^{n+1}:=(w_x^{n+1},w_y^{n+1})^\top\in\partial H_\delta(x^{n+1},x^n)$ with $w_x^{n+1} \in \partial g(x^{n+1}) +\nabla f(x^{n+1}) + 2\delta(x^{n+1}-x^n)$ and $w_y^{n+1} = -2\delta(x^{n+1}-x^n)$. The Lipschitz continuity of $\nabla f$ and using \eqref{eq:ipiano-up} to specify an element from $\partial g(x^{n+1})$ imply
\[ 
  \begin{array}{rcl}
    \norm[2]{w^{n+1}} 
    & \leq & \norm[2]{w_x^{n+1}} + \norm[2]{w_y^{n+1}} \\
    & \leq & \norm[2]{\nabla f(x^{n+1}) - \nabla f(x^n)} + (\frac 1{\alpha_n}+4\delta)\norm[2]{x^{n+1}-x^n} \\
    &      & + \frac{\beta_n}{\alpha_n} \norm[2]{x^n - x^{n-1}} \\
    & \leq & \frac {1}{\alpha_n} ( \alpha_n L_n + 1 + 4\alpha_n\delta)\Delta_{n+1} + \frac{1}{\alpha_n}\beta_n \Delta_{n} \,.
  \end{array}
\]
As $\alpha_n L_n \leq 2(1-\beta_n)\leq 2$ and $\delta\alpha_n = 1-\frac 12\alpha_nL_n-\frac 12\beta_n \leq 1$, setting $b=\frac 7{c_1}$ verifies condition H2, i.e., $\norm[2]{w^{n+1}} \leq b(\Delta_n + \Delta_{n+1})$.
\item In Theorem~\ref{thm:convergence}(\ref{thm:limit-is-critical}) it is proved that there exists a subsequence $(x^{n_j+1})_{j\in\N}$ of $(x^n)_{n\in\N}$ such that $\lim_{j\to\infty} h(x^{n_j+1}) = h(x^*)$. Proposition~\ref{prop:convergence}(\ref{prop:conv-diff-xk}) shows that $\Delta_{n+1}\to 0$ as $n \to \infty$, hence $\lim_{j\to\infty} x^{n_j} = x^*$. As the term $\delta \norm[2]{x-y}^2$ is continuous in $x$ and $y$, we deduce 
\[
  \lim_{j\to\infty} H(x^{n_j+1},x^{n_j}) = \lim_{j\to\infty} h(x^{n_j+1}) + \delta\norm[2]{x^{n_j+1} - x^{n_j}} = H(x^*,x^*) = h(x^*)\,.
\]
\end{itemize}
Now, the abstract convergence Theorem~\ref{thm:conv-abstract} concludes the proof.
\qquad\end{proof}
\bigskip

The next corollary makes use of the fact that semi-algebraic functions (Definition~\ref{def:semi-alg-set}) have the \KL property. 

\begin{corollary}[Convergence of iPiano for semi-algebraic functions]\label{cor:convergence-whole-seq-semi-alg}
Let $h$ be a semi-algebraic function. Then, $H_\delta(x,y)$ is also semi-algebraic. Furthermore, let $(x^n)_{n\in\N}$, $(\delta_n)_{n\in\N}$, $(x^{n+1},x^n)_{n\in\N}$ be as in Theorem~\ref{thm:convergence-whole-seq}. Then the sequence $(x^n)_{n\in\N}$ has finite length, $x^n\to x^*$ as $n\to\infty$, and $x^*$ is a critical point of $h$.
\end{corollary}
\begin{proof}
  As $h$ and $\delta\norm[2]{x-y}$ are semi-algebraic, $H_\delta(x,y)$ is semi-algebraic and has the KL property. Then, Theorem~\ref{thm:convergence-whole-seq} concludes the proof.
\qquad\end{proof}
\bigskip

\subsection{Convergence rate} \label{subsec:nc-conv-rate}

In the following, we are interested in determining a convergence rate with respect to the proximal residual from Definition~\ref{def:prox-res}. Since all preceding estimations are according to $\norm[2]{x^{n+1}-x^n}$ we establish the relation to $\norm[2]{\res(x)}$ first. The following lemmas about the monotonicity and the non-expansiveness of the proximity operator turn out to be very useful for that. Coarsely speaking, Lemma~\ref{lem:proxmono} states that the residual is sub-linearly increasing. Lemma~\ref{lem:proxnonexp} formulates a standard property of the proximal operator.
\begin{lemma}[Proximal monotonicity]\label{lem:proxmono}
  Let $y,z\in \R^N$, and $\alpha >0$. Define the functions
  \[
    p_g(\alpha) := \frac 1\alpha \norm[2]{(I+\alpha \partial g)^{-1}(y-\alpha z)-y}
  \]
  and
  \[
    q_g(\alpha) := \norm[2]{(I+\alpha \partial g)^{-1}(y-\alpha z)-y}.
  \]
  Then, $p_g(\alpha)$ is a decreasing function of $\alpha$, and $q_g(\alpha)$
  increasing in $\alpha$.
\end{lemma}
\begin{proof}
See e.g. \cite[Lemma 1]{Nest13} or \cite[Lemma 4]{Sra12}.
\qquad\end{proof}
\bigskip

\begin{lemma}[Non-expansiveness]\label{lem:proxnonexp}
Let $g$ be a convex function and $\alpha>0$, then, for all $x,y\in \dom g$ we obtain the non-expansiveness of the proximity operator
\begin{equation}
  \norm[2]{(I+\alpha \partial g)^{-1}(x) - (I+\alpha \partial g)^{-1}(y)}
  \leq \norm[2]{ x-y }, \quad \forall x,y\in \R^N.
\end{equation}
\end{lemma}
\begin{proof}
It is a well-known fact. See for example \cite{BC11}.
\qquad\end{proof}
\bigskip

The two preceding lemmas allow us to establish the following relation.
\begin{lemma}\label{lem:rho-upper-bound}
  We have the following bound:
  \begin{equation}
     \sum_{n=0}^N \norm[2]{\res(x^n)} \leq \frac 2{c_1} \sum_{n=0}^N \norm[2]{x^{n+1}-x^n}\,.
  \end{equation}
\end{lemma}
\begin{proof}
  First, we observe the relations $1\leq \alpha \Rightarrow q_g(1)\leq q_g(\alpha)$ and $1\geq \alpha \Rightarrow p_g(1)\leq p_g(\alpha) = \frac 1\alpha q_g(\alpha)$, which are based on Lemma~\ref{lem:proxmono}. Then, invoking the non-expansiveness of the proximity operator (Lemma~\ref{lem:proxnonexp}) we obtain
  \begin{equation}\label{eq:eqA}
  \begin{split}
    \beta_n\norm[2]{x^n-x^{n-1}}
    \geq &\ \norm[2]{x^{n}-\alpha_n\nabla f(x^n) + \beta_n(x^n-x^{n-1}) - (x^{n}-\alpha_{n}\nabla f(x^n)) } \\
    \geq &\ \norm[2]{x^{n+1}-(I+\alpha_{n}\partial g)^{-1}(x^n-\alpha_{n}\nabla f(x^n))}\,.
  \end{split}
  \end{equation}
  This allows us to compute the following lower bound
  \begin{eqnarray*}
     \norm[2]{x^{n+1}-x^n}
     & \geq & \norm[2]{x^{n+1}-x^n} -\beta_n\norm[2]{x^n-x^{n-1}} \\
     &      & +\, \norm[2]{x^{n+1}-(I+\alpha_{n} \partial g)^{-1}(x^n-\alpha_{n}
                                                              \nabla f(x^n))} \\
     & \geq & \norm[2]{x^{n}-(I+\alpha_{n} \partial g)^{-1}(x^n-\alpha_{n}
                               \nabla f(x^n))} -\beta_n\norm[2]{x^n-x^{n-1}}  \\
     & \geq & \min(1,\alpha_{n})\norm[2]{\res(x^{n})} -\norm[2]{x^n-x^{n-1}}  \\
     & \geq & c_1\norm[2]{\res(x^{n})} -\norm[2]{x^n-x^{n-1}}\,,
  \end{eqnarray*}
  where the first inequality arises from adding zero and using~\eqref{eq:eqA}, the second uses the triangle inequality, the next one applies Lemma~\ref{lem:proxmono} and $\beta_n<1$. Now, summing both sides from $n=0,\ldots,N$ and using $x^{-1}=x^0$ the statement easily follows.
\qquad\end{proof}
\bigskip

Next, we prove a global $\mathcal{O}(1/n)$ convergence rate for $\norm[2]{x^{n+1}-x^n}^2$ and the residuum $\norm[2]{\res(x^{n})}^2$ of the algorithm. The residuum provides an error measure of being a fixed point and hence a critical point of the problem. We first define the error $\mu_N$ to be the smallest squared $\ell_2$ norm of successive iterates and, analogously, the error $\mu^\prime_N$
\[
\mu_N := \min_{0\leq n\leq N} \norm[2]{x^{n}-x^{n-1}}^2\quad \text{and}\quad
\mu^\prime_N := \min_{0\leq n\leq N} \norm[2]{\res(x^{n})}^2\,.
\]
\begin{theorem}\label{thm:conv-rate}
  Algorithm~\ref{alg:ipiano-general} guarantees that for all $N\geq 0$
  \[
    \mu^\prime_N \leq \frac{2}{c_1}\mu_N \quad\text{and}\quad \mu_N \leq c_2^{-1}\frac{h(x^0)-\underline h}{N+1}\,.
  \]
\end{theorem}
\begin{proof} In view of Proposition~\ref{prop:convergence}(\ref{prop:conv-hd}), and the definition of $\gamma_N$ in \eqref{eq:def-delta-gamma}, summing up both sides of \eqref{eq:hd-descent} for $n=0,\ldots,N$ and using that $\delta_N>0$ from \eqref{eq:def-delta-gamma} we obtain
  \[
    \underline h
    \leq h(x^0) - \sum_{n=0}^N\gamma_{n}\norm[2]{x^{n}-x^{n-1}}^2 \\
    \leq h(x^0) - (N+1) \min_{0\leq n\leq N} \gamma_{n} \mu_N \,.
  \]
  As it is $\gamma_n > c_2$, a simple rearrangement invoking Lemma~\ref{lem:rho-upper-bound} concludes the proof.
\qquad\end{proof}
\bigskip

\begin{remark}
  The convergence rate $\mathcal O(1/N)$ for the squared $\ell_2$ norm of our error measures is equivalent to stating a convergence rate $\mathcal O(1/\sqrt N)$ for the error in the $\ell_2$ norm.
\end{remark}

\begin{remark}
  A similar result can be found in \cite{Nest13} for the case $\beta=0$.
\end{remark}

\section{Numerical experiments} \label{sec:numexp}
In all the following experiments, let $\img, \noisy\in \R^N$ be vectors of dimension
$N\in\N$, where $N$ depends on the respective problem. In the case of an image
$N$ is the number of pixels.

\subsection{Ability to overcome spurious stationary points}

Let us present some of the qualitative properties of the proposed
algorithm. For this, we consider to minimize the following simple
problem
\begin{equation}\label{eq:2d-example}
\min_{x\in\R^N} h(x) := f(x) + g(x)\,, \quad f(x) = \frac 1 2 \sum_{i=1}^N
\log(1+\mu(x_i-\noisy_i)^2)\,,\quad g(x)=\lambda \|x\|_1\,,
\end{equation}
where $x$ is the unknown vector, $\noisy$ is some given vector, and $\lambda,\mu > 0$
are some free parameters. A contour plot and the energy landscape of $h$ in the
case of $N=2$, $\lambda=1$, $\mu=100$, and $\noisy=(1,1)^\top$ is depicted in
Figure~\ref{fig:2d-example}. It turns out that the function $h$ has four
stationary points, i.e. points $\bar x$, such that $0 \in \nabla f(\bar x) +
\partial g(\bar x)$. These points are marked by small black diamonds. 
\begin{figure}
\begin{center}
  \subfigure[Contour plot of $h(x)$]{\includegraphics[width=0.45\linewidth]{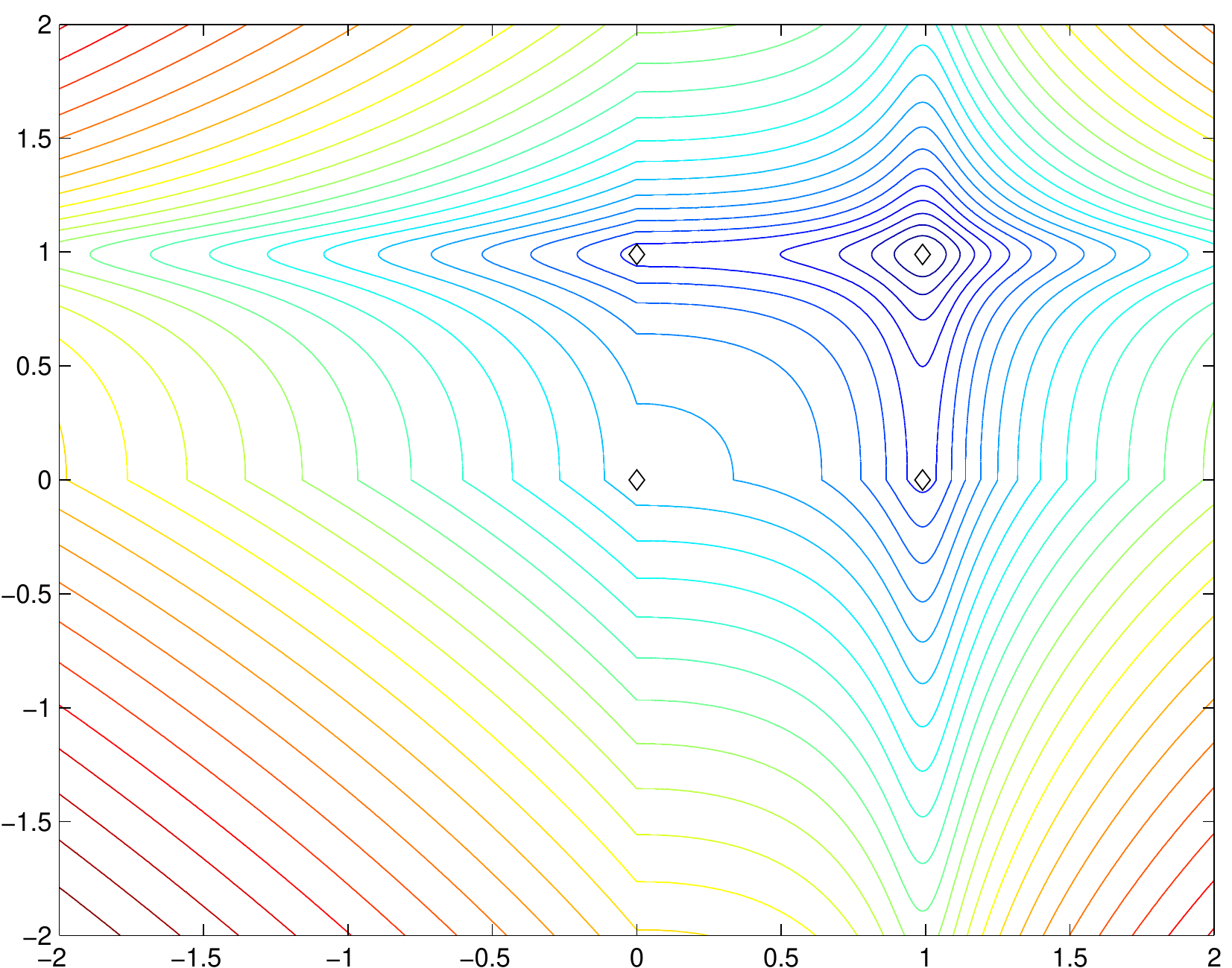}}\hfill
  \subfigure[Energy landscape of $h(x)$]{\includegraphics[width=0.45\linewidth]{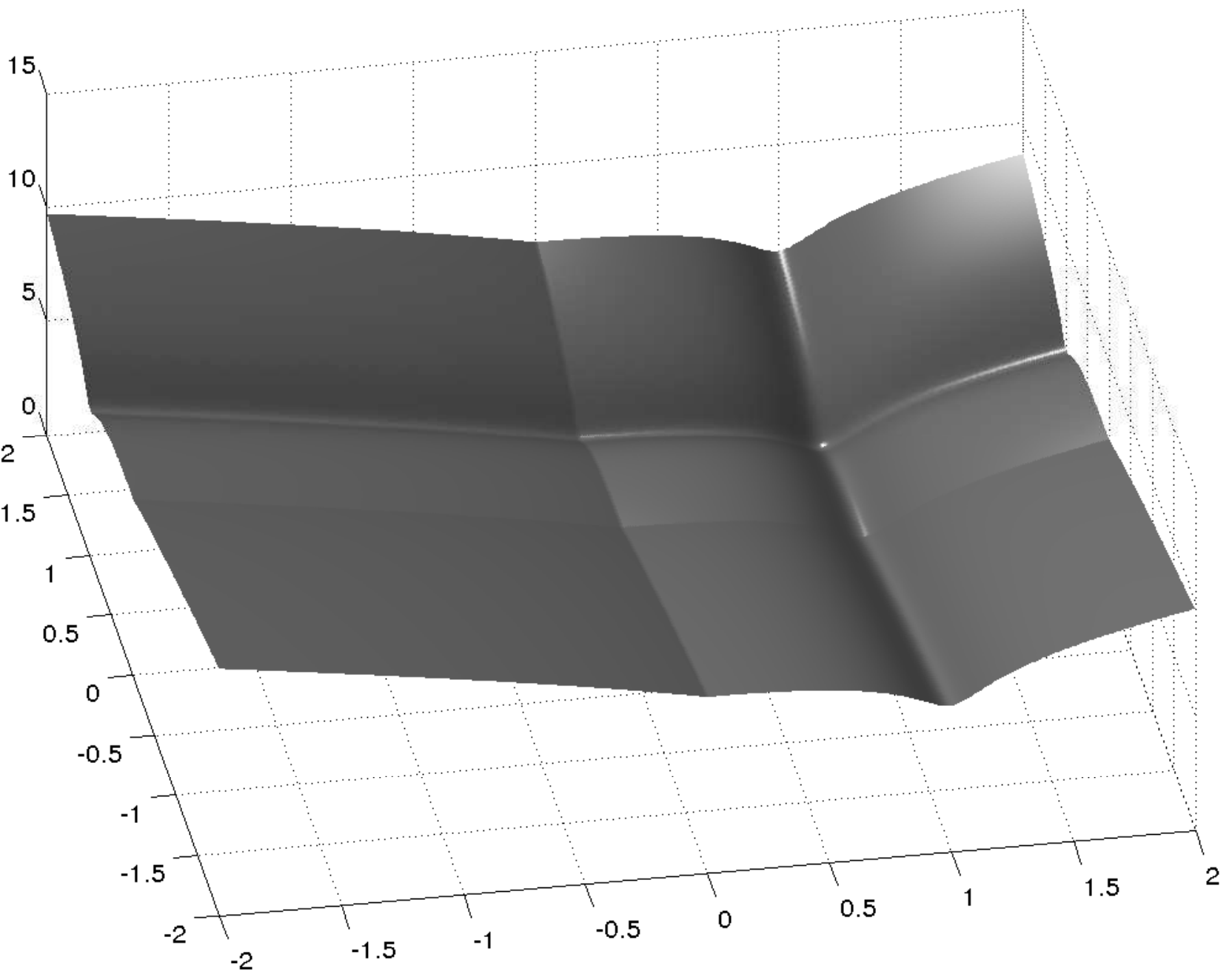}}
\end{center}
\caption{Contours plot (left) and energy landscape (right) of the
  non-convex function $h$ shown in~\eqref{eq:2d-example}. The four
  diamonds mark stationary points of the function
  $h$.}\label{fig:2d-example}
\end{figure}
Clearly the function $f$ is non-convex but has a Lipschitz continuous gradient
with components
\[
\nabla f(x)_i = \mu \frac{x_i-\noisy_i}{1+\mu(x_i-\noisy_i)^2}\,
\]
The Lipschitz constant of $\nabla f$ is easily computed as $L=\mu$. The function
$g$ is non-smooth but convex and the proximal operator with respect to $g$ is
given by the well-known shrinkage operator
\begin{equation}\label{softshrinkage}
(I + \alpha \partial g)^{-1}(y) = \max(0, |y| -
\alpha\lambda)\cdot\mathrm{sgn}(y)\,,
\end{equation}
where all operations are understood component-wise. Let us test the performance
of the proposed algorithm on the example shown in Figure~\ref{fig:2d-example}.
We set $\alpha=2(1-\beta)/L$. Figure~\ref{fig:2d-example-results} shows the
results of using the iPiano algorithm for different settings of the
extrapolation factor $\beta$. We observe that iPiano with $\beta=0$ is strongly
attracted by the closest stationary points while switching on the inertial term
can help to overcome the spurious stationary points. The reason for this desired
property is that while the gradient might vanish at some points, the inertial
term $\beta(x^n-x^{n-1})$ is still strong enough to drive the sequence out of
the stationary region. Clearly, there is no guarantee that iPiano always avoids
spurious stationary points. iPiano is not designed to find the global optimum. 
However, our numerical experiments suggest that in many cases, iPiano finds lower 
energies than the respective algorithm without inertial term. A similar observation 
about the Heavy-ball method is described in
\cite{Bertsekas99}.
\begin{figure}
\begin{center}
  {\includegraphics[width=0.24\linewidth]{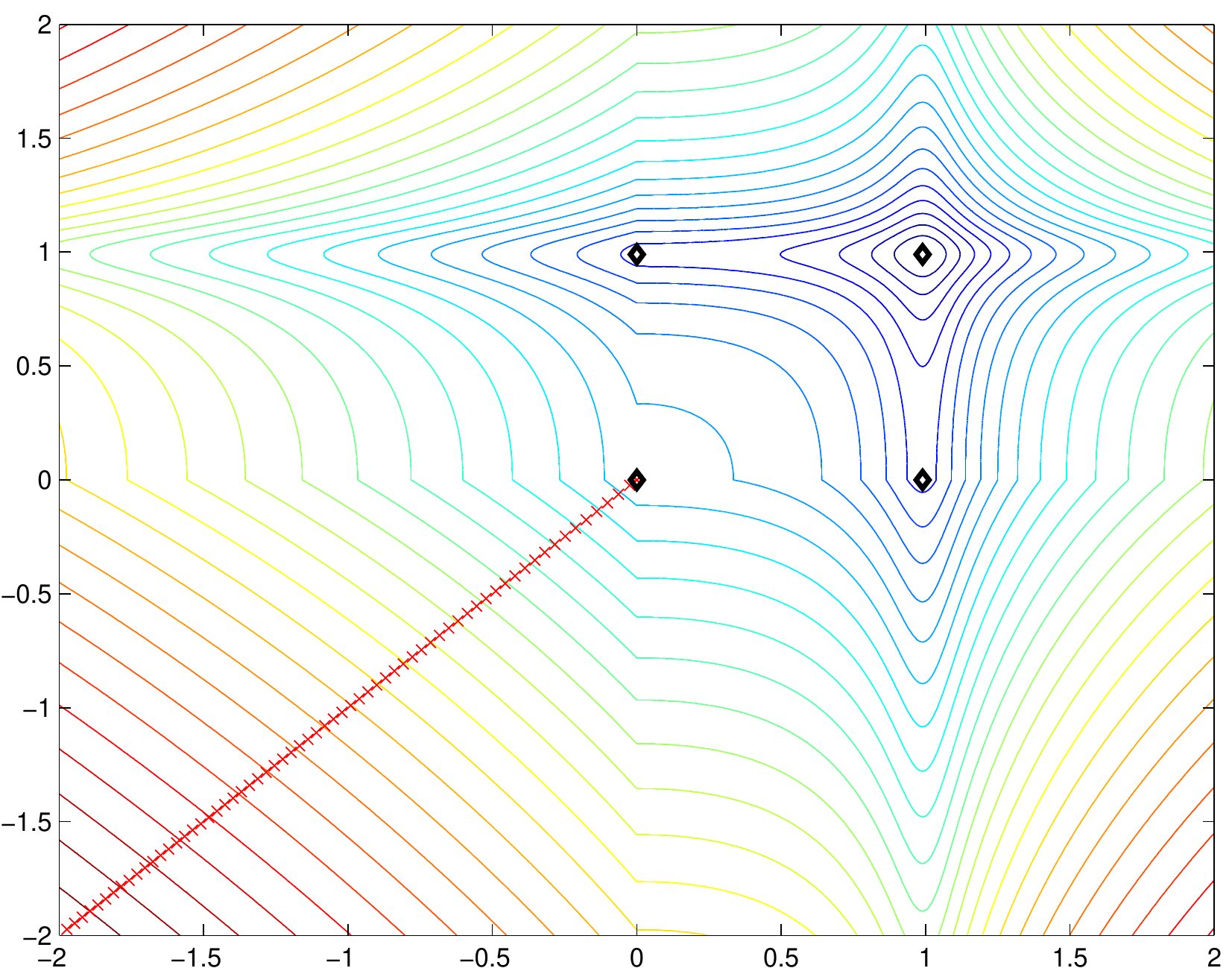}}\hfill
  {\includegraphics[width=0.24\linewidth]{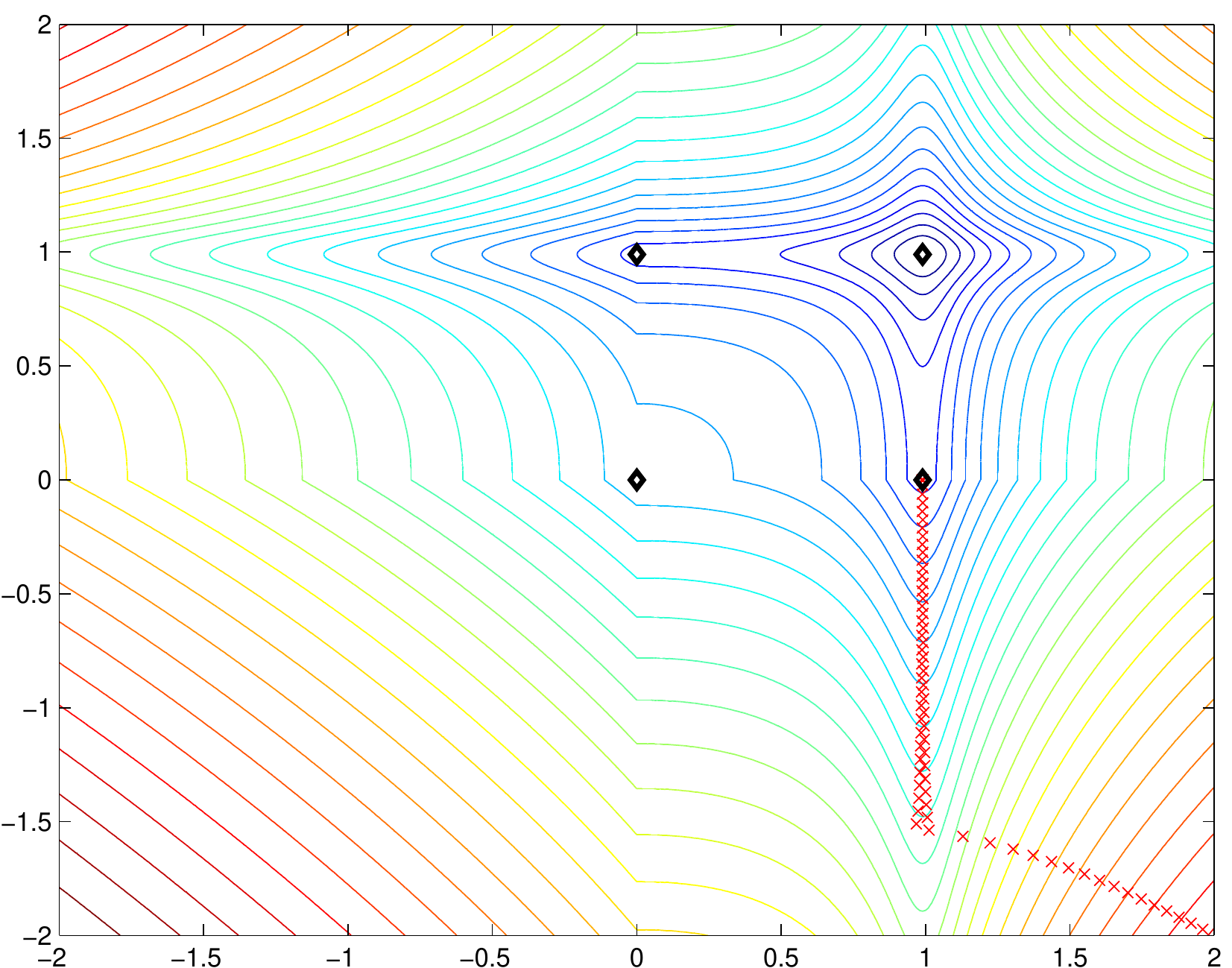}}\hfill
  {\includegraphics[width=0.24\linewidth]{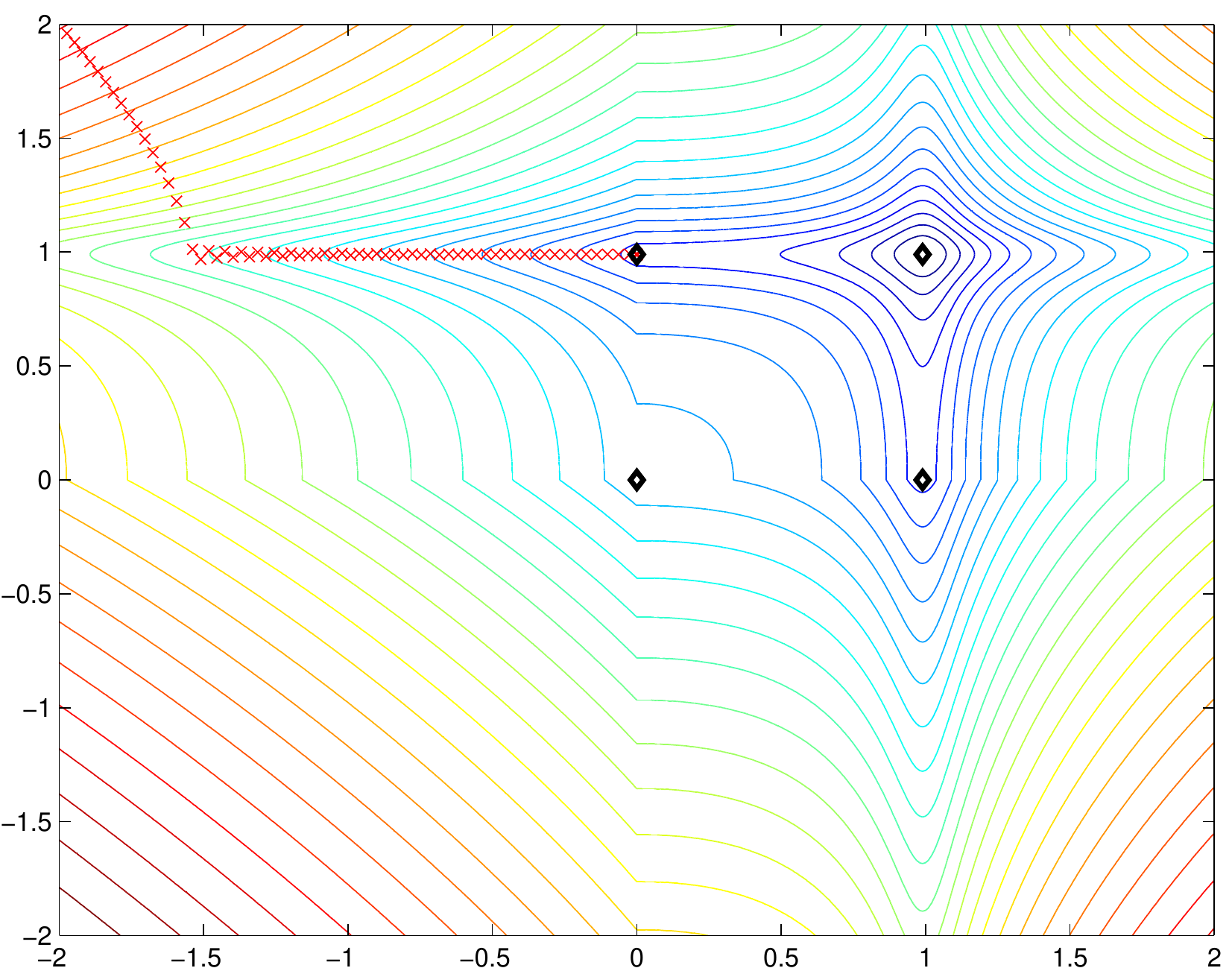}}\hfill
  {\includegraphics[width=0.24\linewidth]{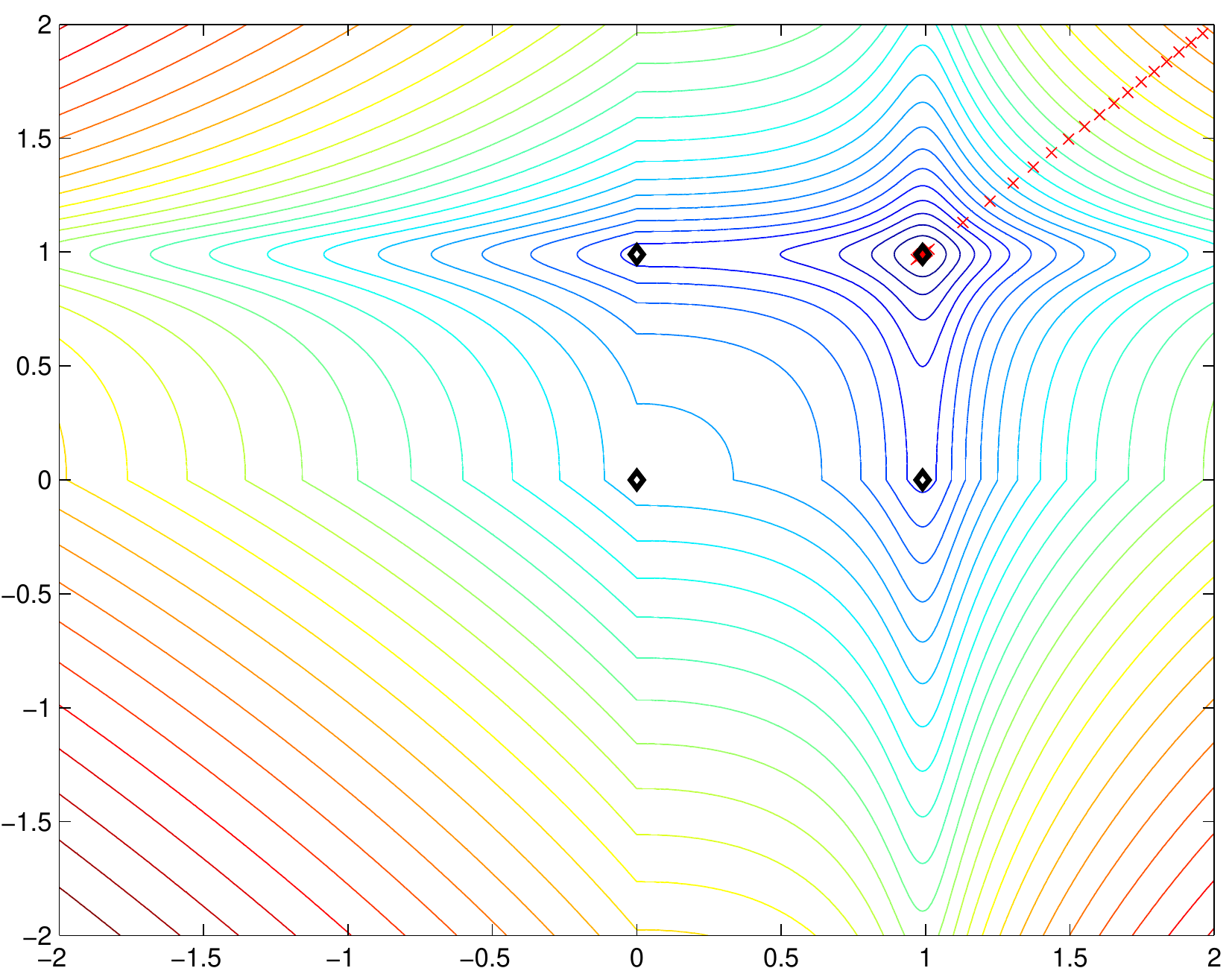}}\\[2mm]
  {\includegraphics[width=0.24\linewidth]{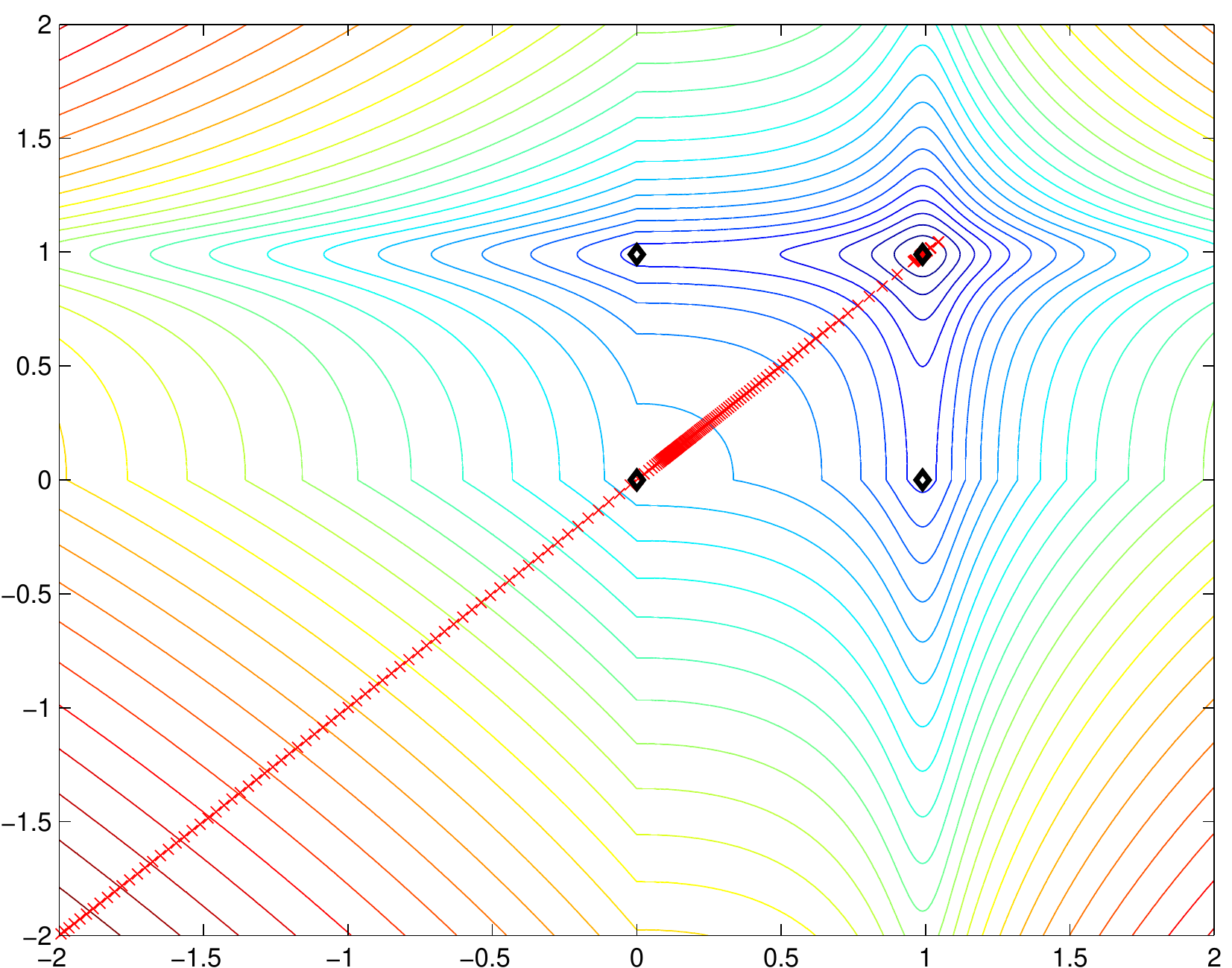}}\hfill
  {\includegraphics[width=0.24\linewidth]{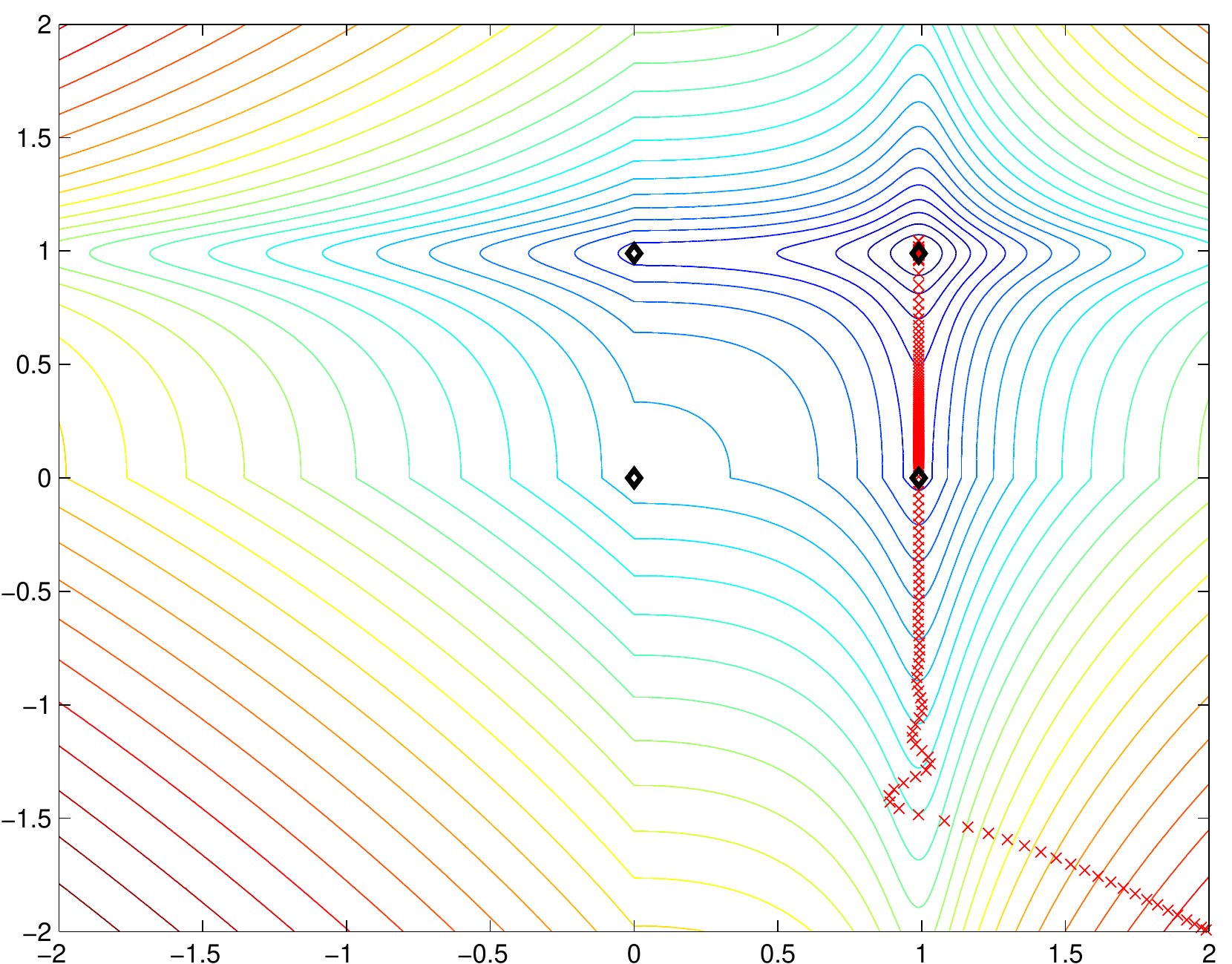}}\hfill
  {\includegraphics[width=0.24\linewidth]{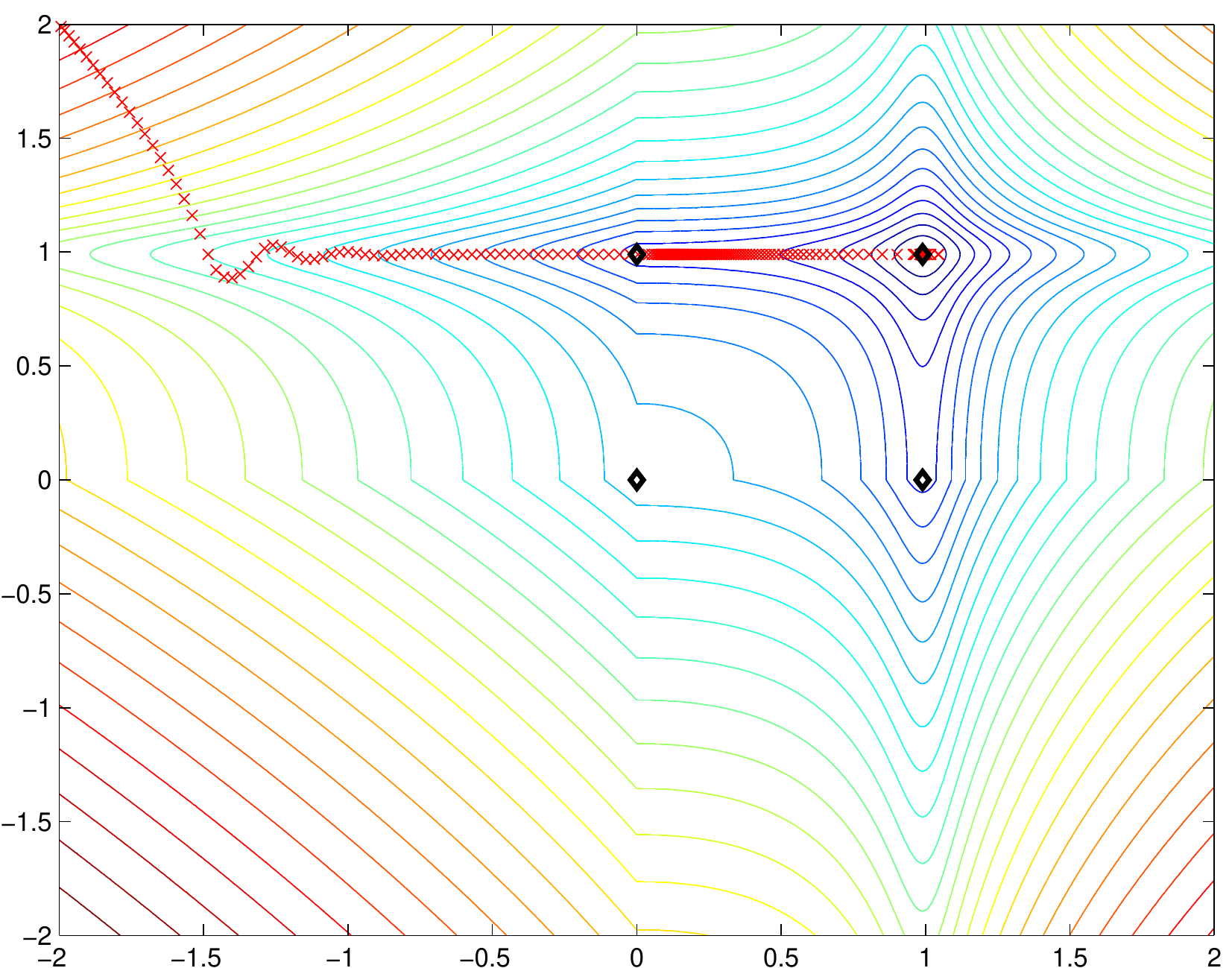}}\hfill
  {\includegraphics[width=0.24\linewidth]{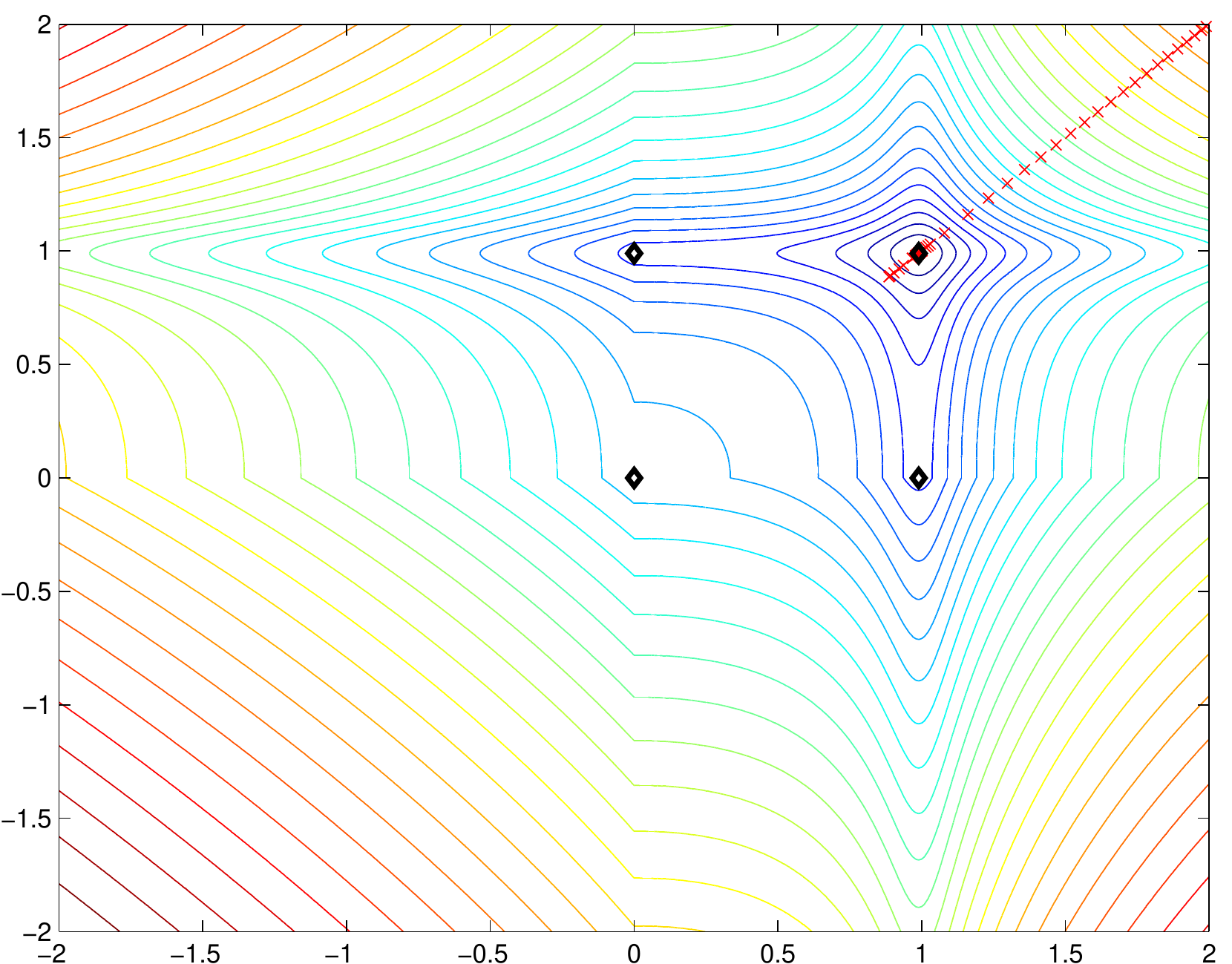}}
\end{center}
\caption{The first row shows the result of the iPiano algorithm for
  four different starting points when using $\beta=0$, the second row
  shows the results when using $\beta=0.75$. While the algorithm
  without inertial term gets stuck into unwanted local stationary
  points in three of four cases, the algorithm with inertial term
  always succeeds to converge to the global optimum.}
  \label{fig:2d-example-results}
\end{figure}

\subsection{Image processing applications}
It is well-known that non-convex regularizers are better models for many image processing and computer vision problems, see e.g.~\cite{BlR96, GG84, Huang1999_Statistics, RothBlack}. However, convex models are still preferred over non-convex ones, since they can be efficiently optimized using convex optimization algorithms. In this section, we demonstrate the applicability of the proposed algorithm to solve a class of non-convex regularized variational models. We present examples for natural image denoising, and linear diffusion based image compression. We show that iPiano can be easily adapted to all these problems and yields state-of-the-art results.

\subsubsection{Student-t regularized image denoising}

In this subsection, we investigate the task of natural image denoising. For this we exploit an optimized MRF (Markov random field) model, which is learned in following~\cite{ChenGCPR13}, and make use of the iPiano algorithm to solve it. In order to evaluate the performance of iPiano, we compare it to the well-known bound constrained limited memory quasi Newton method (L-BFGS)~\cite{BFGS}~\footnote{We make use of the implementation distributed at \url{http://www.cs.toronto.edu/~liam/software.shtml}.}. As an error measure, we use the energy difference
\begin{equation}
\cE^n = h^n - h^*\,,
\end{equation}
where $h^n$ is the energy of the current iteration $n$ and $h^*$ is the energy of the true solution. Clearly, this error measure makes sense only when different algorithms can achieve the same true energy $h^*$ which is in general wrong for non-convex problems. In our image denoising experiments, however, we find, that all tested algorithms find the same solution, independent of the initialization. This can be explained by the fact that the learning procedure~\cite{ChenGCPR13} also delivers models that are relatively easy to optimize, since otherwise they would have resulted in a bad training error. In order to compute a true energy $h^*$, we run the iPiano algorithm with a proper $\beta$ (e.g., $\beta = 0.8$) for enough iterations ($\sim$1000 iterations).  We run all the experiments in Matlab on a 64-bit Linux server with 2.53GHz CPUs.

The MRF image denoising model based on learned filters is formulated as
\begin{equation}\label{mrfl2}
  \min\limits_{\img \in \R^N}
  \sum_{i=1}^{N_f}\vartheta_i\Phi(K_i\img) + g_{1,2}(\img, \noisy)\,,
\end{equation}
where $\img$ and $\noisy \in \R^N$ denote the sought solution and the noisy input image respectively, $\Phi$ is the non-convex penalty function, $\Phi(K_i\img) = \sum_p \varphi((K_i\img)_p)$, $K_i$ are learned, linear operators with the corresponding weights $\vartheta_i$, and $N_f$ is the number of the filters.  The linear operators $K_i$ are implemented as 2D convolutions of the image $u$ with small (e.g. $7\times 7$) filter kernels $k_i$, i.e. $K_iu=k_i*u$. The function $g_{1,2}$ is the data term, which depends on the respective problem. In the case of Gaussian noise, $g_{1,2}$ is given as
\[
g_2(\img, \noisy) = \frac{\lambda}{2}\|\img- \noisy \|_2^2\,,
\]
and for the impulse noise (e.g., salt \& pepper noise), $g_{1,2}$ is given as
\[
g_1(\img, \noisy) = \lambda\|\img- \noisy \|_1\,.
\]
The parameter $\lambda > 0$ is used to define the tradeoff between regularization and data fitting.

In this paper, we consider the following non-convex penalty function, which is derived from the Student-t distribution:
\begin{equation}
\varphi(t) = \text{log}(1+t^2)\,.
\end{equation}

Concerning the filters $k_i$, for the $\ell_2$ model (MRF-$\ell_2$), we make use of the filters learned in~\cite{ChenGCPR13}, by using a bi-level learning approach. The filters are shown in Figure~\ref{fig:filters}(a) together with the corresponding weights $\vt_i$. For the MRF-$\ell_1$ denoising model, we employ the same bi-level learning algorithm to train a set of optimal filters specialized for the $\ell_1$ data term and input images degraded by salt \& pepper noise.  Since the bi-level learning algorithms requires a twice continuously differentiable model we replace the $\ell_1$ norm by a smooth approximation during training. The learned filters for the MRF-$\ell_1$ model together with the corresponding weights $\vt_i$ are shown in Figure~\ref{fig:filters}(b).

\begin{figure}[t!]
\begin{center}
  \subfigure[Learned filters for the MRF-$\ell_2$ model]{\includegraphics[width=1\linewidth]{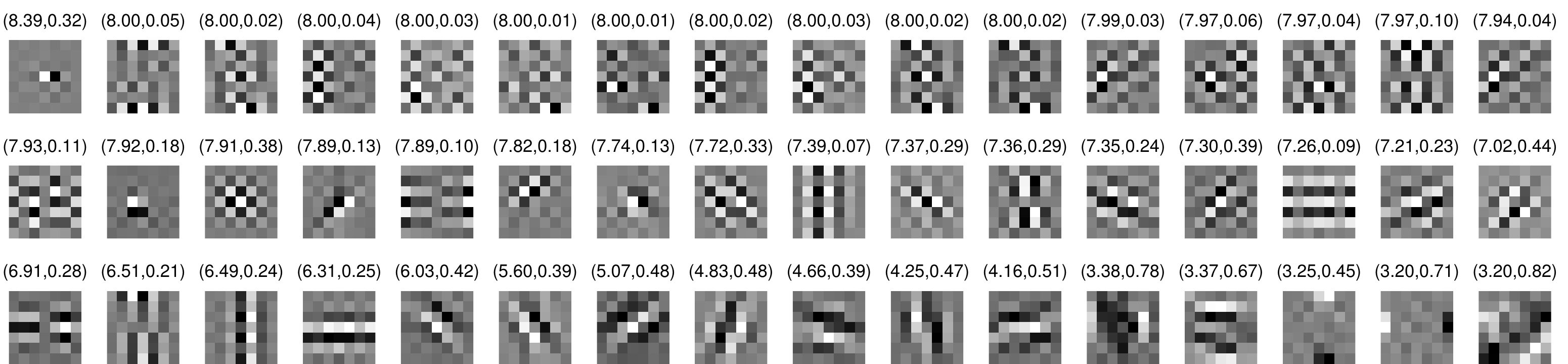}}\\
  \subfigure[Learned filters for the MRF-$\ell_1$ model]{\includegraphics[width=1\linewidth]{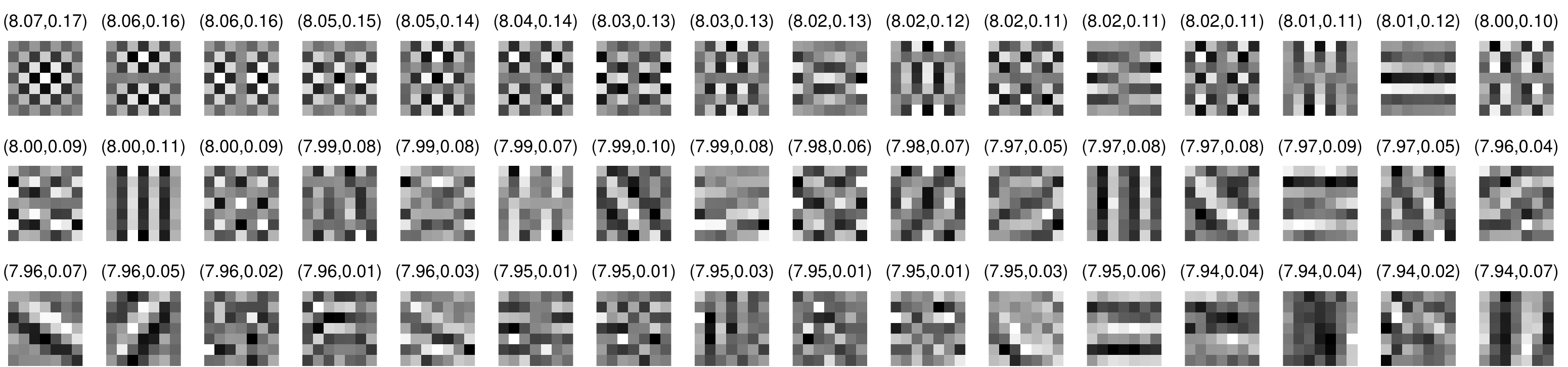}}
\end{center}
\caption{48 learned filters of size $7 \times 7$ for two different MRF
  denoising models. The first number in the bracket is the weights
  $\vt_i$, and the second one is the norm $\|k_i\|_2$ of the
  filters.}\label{fig:filters}
\end{figure}
\begin{figure}[t!]
\begin{center}
  \subfigure[Clean image]{\includegraphics[width=0.33\linewidth]{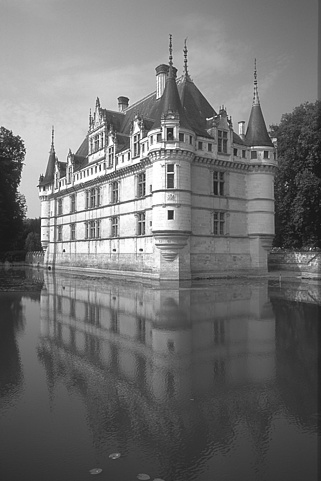}}\hfill
  \subfigure[Noisy image ($\sigma = 25$)]{\includegraphics[width=0.33\linewidth]{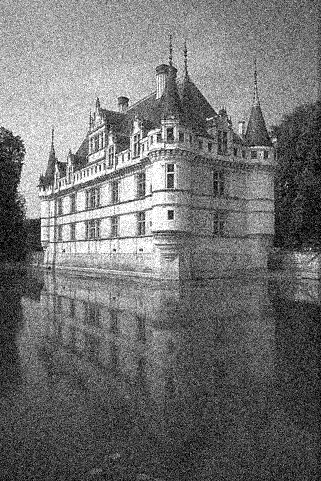}}\hfill
  \subfigure[Denoised image]{\includegraphics[width=0.33\linewidth]{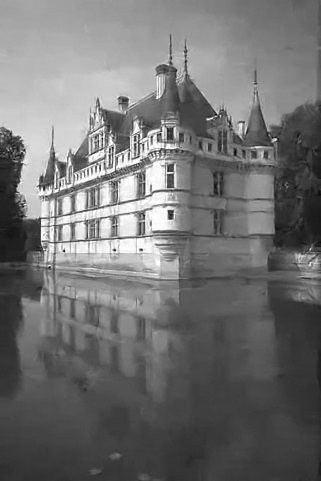}}\hfill
\end{center}
\caption{Natural image denoising by using Student-t regularized MRF model (MRF-$\ell_2$). The noisy version
is corrupted by additive zero mean Gaussian noise with $\sigma = 25$.}\label{fig:l2denoising}
\end{figure}

\begin{figure}[h!]
\begin{center}
  \subfigure[Clean image]{\includegraphics[width=0.45\linewidth]{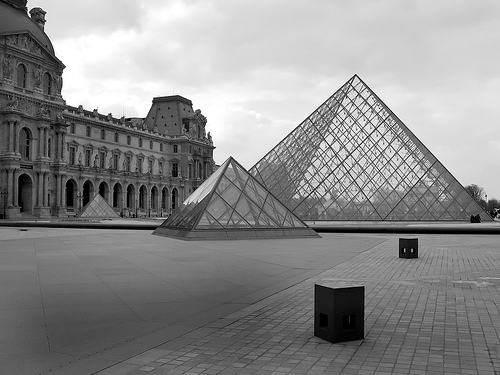}}\hfill
  \subfigure[Noisy image (25\% salt \& pepper noise)]{\includegraphics[width=0.45\linewidth]{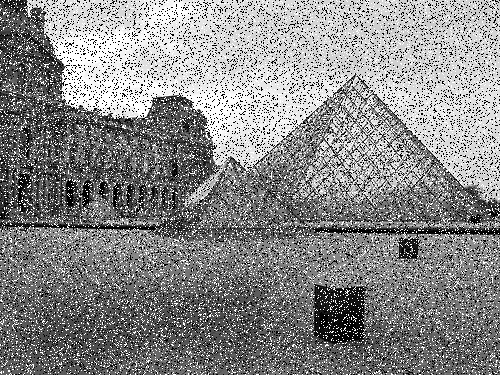}}\\
  \subfigure[Denoised image]{\includegraphics[width=0.45\linewidth]{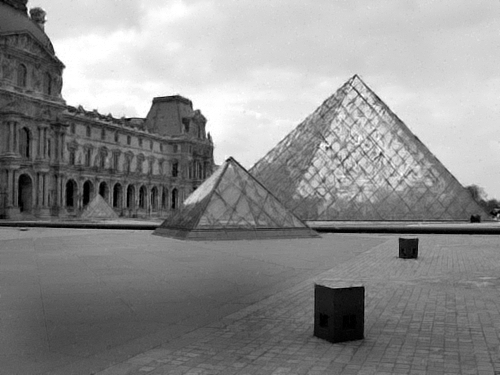}}
\end{center}
\caption{Natural image denoising in the case of impulse noise by using the MRF-$\ell_1$ model. The noisy version
is corrupted by 25\% salt \& pepper noise. }\label{fig:l1denoising}
\end{figure}

Let us now explain how to solve~\eqref{mrfl2} using the iPiano
algorithm.  Casting \eqref{mrfl2} in the form of
\eqref{eq:problem-class}, we see that $f(\img) =
\sum_{i=1}^{N_f}\vartheta_i\Phi(K_i\img)$ and $g(\img) =
g_{1,2}(\img, \noisy)$. Thus, we have
\[
\nabla f(\img) = \sum_{i=1}^{N_f}\vartheta_i K_i^\top \Phi'(K_i\img)\,,
\]
where $\Phi'(K_i\img) = [\varphi'((K_i\img)_1)\,, \varphi'((K_i\img)_2), \dots, \varphi'((K_i\img)_p)]^\top$
and $\varphi'(t) = {2t}/{(1+t^2)}$.
The proximal map with respect to $g$ simply poses point-wise operations.
For the case of $g_2$, it is given by
\[
\img = (I + \alpha \partial g)^{-1}(\hat \img) \Longleftrightarrow
\img_p = \frac{\hat \img_p + \alpha\lambda\noisy_p}{1 + \alpha\lambda}\,, \quad p=1...N
\]
and for the function $g_1$, it is given by the well-known soft
shrinkage operator~\eqref{softshrinkage}, which in case of the
MRF-$\ell_1$ model becomes
\begin{align*}
\img &= (I + \alpha \partial g)^{-1}(\hat \img) \Longleftrightarrow\\
\img_p &= \max(0, ~|\hat \img_p - \noisy_p| - \alpha\lambda)\cdot \text{sgn}(\hat \img_p - \noisy_p) + \noisy_p\,,  \quad p=1...N\,.
\end{align*}
Now, we can make use of our proposed algorithm to solve the non-convex
optimization problems.  In order to evaluate the performance of
iPiano, we compare it to L-BFGS. To use L-BFGS, we merely need the
gradient of the objective function with respect to $\img$. For the
MRF-$\ell_2$ model, calculating the gradients is straightforward.
However, in the case of the MRF-$\ell_1$ model, due to the non-smooth
function $g$, we cannot directly use L-BFGS. Since L-BFGS can easily
handle box constraints, we can get rid of the non-smooth function
$\ell_1$ norm by introducing two box constraints.
\begin{lemma}\label{equivalencemrfl1}
  The MRF-$\ell_1$ model can be equivalently written as the
  bound-constraint problem:
\begin{align}\label{mrfl1}
  \min_{w, v} \sum_{i=1}^{N_f}\vt_i\Phi(K_i(w + v)) + \lambda \, 1^\top (v - w)
  \quad \mathrm{s.t.} \quad w \leq \noisy/2, ~v \geq \noisy/2\,.
\end{align}
\end{lemma}

\begin{proof}
  It is well-know that the $\ell_1$ norm $\|\img - \noisy\|_1$ can be
  equivalently expressed as
\begin{equation*}
\|\img - \noisy\|_1 = \min_{t} 1^\top t\,, \quad\text{s.t.} \quad t \geq \img - \noisy\,, \quad t \geq -\img + \noisy\,,
\end{equation*}
where $t \in \R^N$ and the inequalities are understood pointwise.
Letting $w = (\img - t)/2 \in \R^N$, and $v = (\img + t)/2 \in \R^N$,
we find $\img = w + v$ and $t = v - w$. Substituting $\img$ and $t$
back into~\eqref{mrfl2} while using the above formulation of the
$\ell_1$ norm yields the desired transformation.
\qquad\end{proof}
\bigskip

\begin{figure}
\begin{center}
  \subfigure[MRF-$\ell_2$ model]{\includegraphics[width=0.49\linewidth]{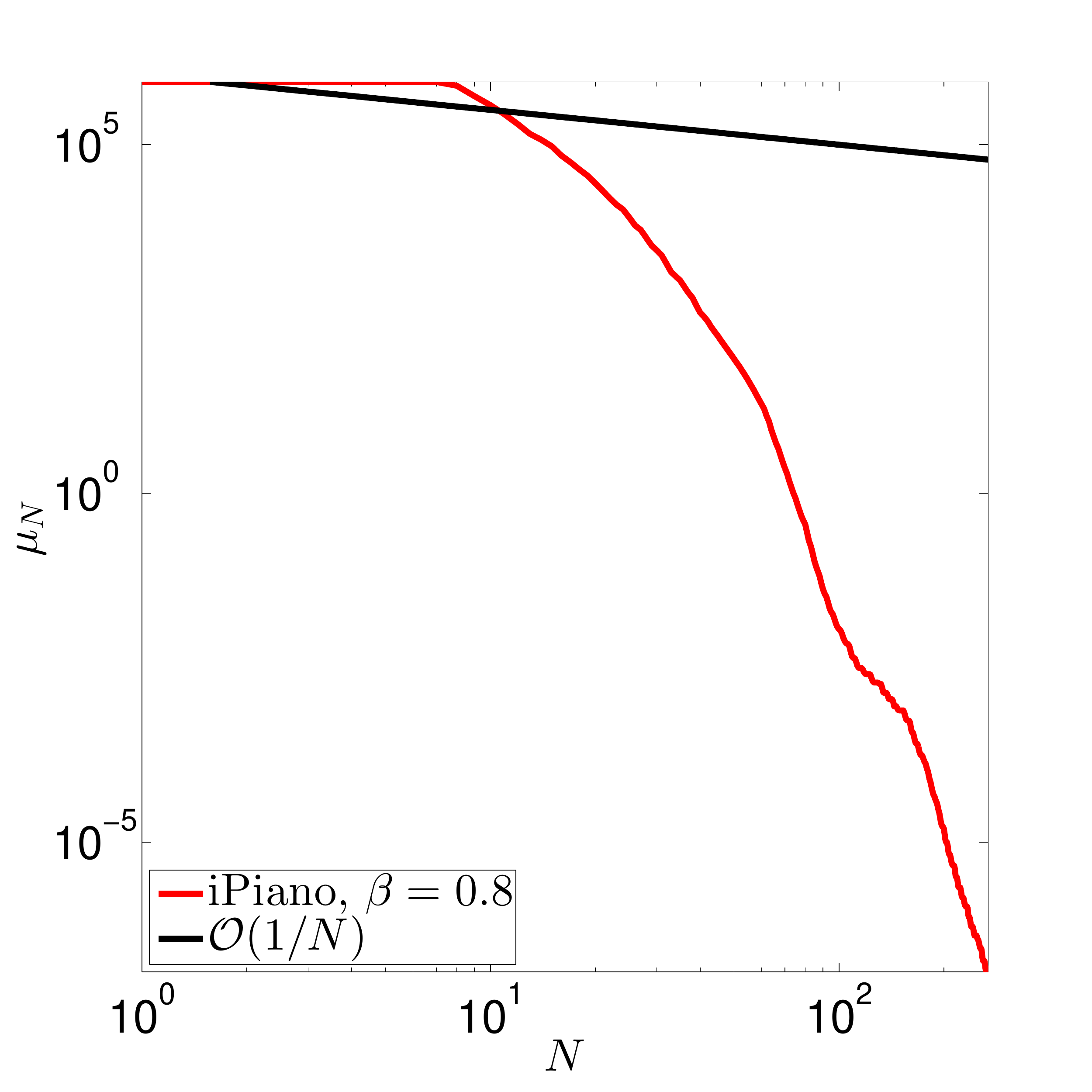}}\hfill
  \subfigure[MRF-$\ell_1$ model]{\includegraphics[width=0.49\linewidth]{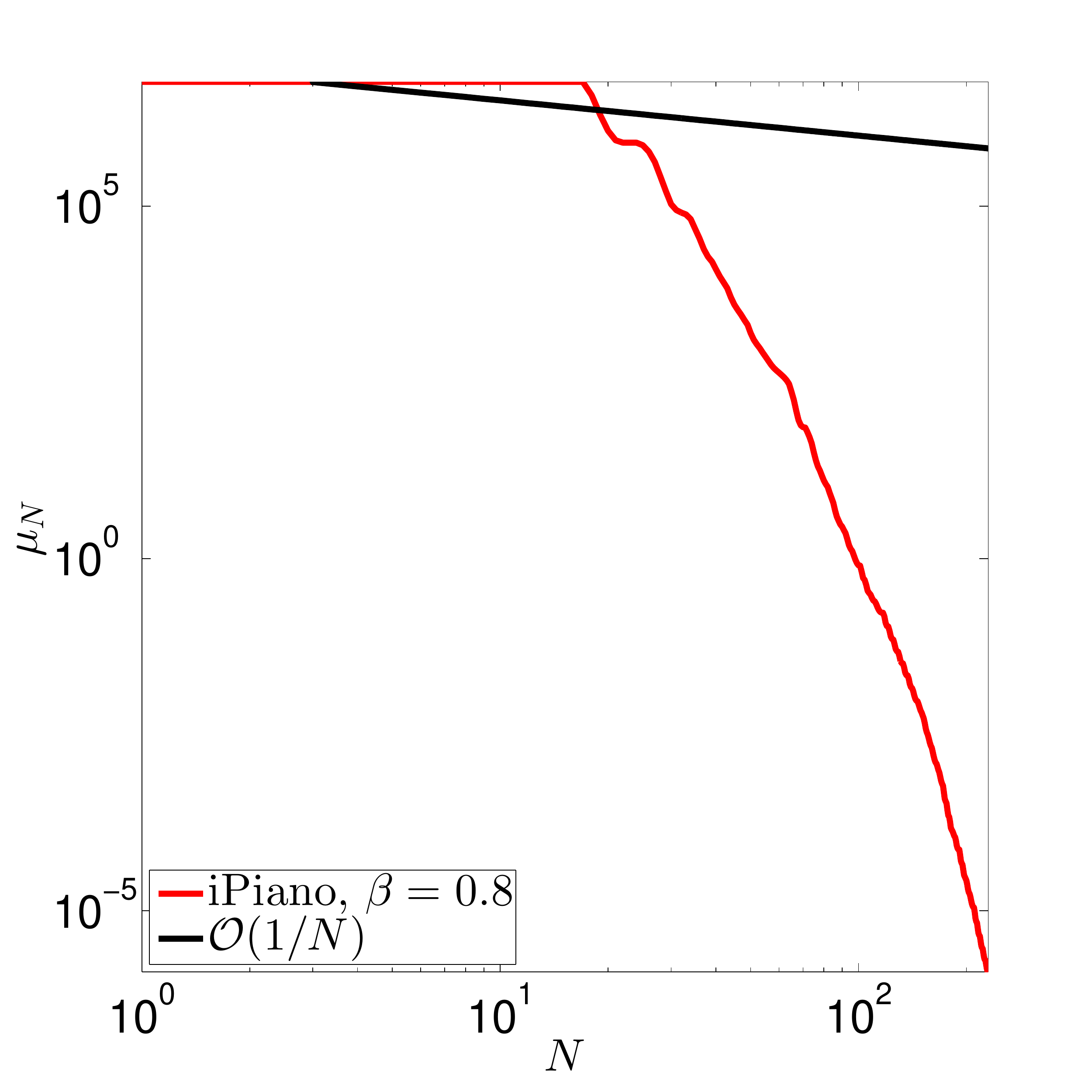}}
\end{center}
\caption{Convergence rates for the MRF-$\ell_2$ and -$\ell_1$
  models. The figures plots the minimal residual norm $\mu_N$ which
  also bounds the proximal residual $\mu'_N$. Note that the empirical
  convergence rate is much faster compared to the worst case rate (See
  Theorem~\ref{thm:conv-rate}).}\label{fig:rates}
\end{figure}

\begin{table}[t]
\centering
\begin{tabular}{|c||cccccc|c||c|c|}
\hline
    &\multicolumn{7}{c||}{iPiano with different $\beta$} &\multicolumn{2}{c|}{L-BFGS}\\
\cline{1-10}
$\text{tol}$ & 0.00 & 0.20 & 0.40 & 0.60 & 0.80 & 0.95 & $T_1$(s) & iter. & $T_2$(s)\\
\cline{1-10}
$10^{3}$ &260 &182 &116 &66 &56 &214 &34.073 &43 &18.465\\
$10^{2}$ &372 &256 &164 &94 &67 &257  &40.199 &55 &22.803\\
$10^{1}$ &505 &344 &222 &129 &79 &299 &47.177 &66 &27.054\\
$10^{0}$ &664 &451 &290 &168 &98 &342 &59.133 &79 &32.143\\
$10^{-1}$ &857 &579 &371 &216 &143 &384 &85.784 &93 &36.926\\
$10^{-2}$ &1086 &730 &468 &271 &173 &427 &103.436 &107 &41.939\\
$10^{-3}$ &1347 &904 &577 &338 &199 &473 &119.149 &124 &48.272\\
$10^{-4}$ &1639 &1097 &697 &415 &232 &524 &138.416 &139 &53.290\\
$10^{-5}$ &1949 &1300 &827 &494 &270 &569 &161.084 &154 &58.511\\
\cline{1-10}
\end{tabular}
\caption{The number of iterations and the run time necessary for reaching the corresponding
  error for iPiano and L-BFGS to solve the MRF-$\ell_2$ model.
  $T_1$ is the run time of iPiano with $\beta = 0.8$ and $T_2$ shows the run time of L-BFGS.}\label{table:mrfl2}
\end{table}

\begin{table}[t]
\centering
\begin{tabular}{|c||cccccc|c||c|c|}
\hline
    &\multicolumn{7}{c||}{iPiano with different $\beta$} &\multicolumn{2}{c|}{L-BFGS}\\
\cline{1-10}
$\text{tol}$ & 0.00 & 0.20 & 0.40 & 0.60 & 0.80 & 0.95 & $T_1$(s) & iter. & $T_2$(s)\\
\cline{1-10}
$10^{3}$ &390 &272 &174 &96 &64 &215 &43.709 &223 &102.383\\
$10^{2}$ &621 &403 &256 &145 &77 &260 &53.143 &246 &112.408\\
$10^{1}$ &847 &538 &341 &195 &96 &304 &65.679 &265 &121.303\\
$10^{0}$ &1077 &682 &433 &247 &120 &349 &81.761 &285 &130.846\\
$10^{-1}$ &1311 &835 &530 &303 &143 &395 &97.060 &298 &136.326\\
$10^{-2}$ &1559 &997 &631 &362 &164 &440  &111.579 &311 &141.876\\
$10^{-3}$ &1818 &1169 &741 &424 &185 &485 &126.272 &327 &148.945\\
$10^{-4}$ &2086 &1346 &853 &489 &208 &529 &142.083 &347 &157.956\\
$10^{-5}$ &2364 &1530 &968 &557 &233 &575 &159.493 &372 &169.674\\
\cline{1-10}
\end{tabular}
\caption{The number of iterations and the run time necessary for reaching the corresponding
  error for iPiano and L-BFGS to solve the MRF-$\ell_1$ model.
  $T_1$ is the run time of iPiano with $\beta = 0.8$ and $T_2$ shows the run time of L-BFGS.
}\label{table:mrfl1}
\end{table}

Figure~\ref{fig:l2denoising} and Figure~\ref{fig:l1denoising}
respectively show a denoising example using the MRF-$\ell_2$ model,
and the MRF-$\ell_1$ model.  In both experiments, we use the iPiano
version with backtracking (Algorithm \ref{alg:ipiano-BTL}) with the
following parameter settings:
\[
L_{-1} = 1, ~\eta = 1.2, ~\alpha_n = 1.99(1-\beta)/L_n\,,
\]
where $\beta$ is a free parameter to be evaluated in the
experiment. In order to make use of possible larger step sizes in
practice, we use a following trick: when the inequality
\eqref{eq:BT-up-cond-lip} is fulfilled, we decrease the evaluated
Lipschitz constant $L_n$ slightly by setting $L_n = L_n/1.05$.

For the MRF-$\ell_2$ denoising experiments, we initialized $\img$
using the noisy image itself, however, for the MRF-$\ell_1$ denoising
model, we initialized $\img$ using a zero image. We found that this
initialization strategy usually gives good convergence behavior for
both algorithms. For both denoising examples, we run the algorithms
until the error $\cE^n$ decreases to a certain predefined threshold
$\text{tol}$. We then record the required number of iterations and the run
time.  We summarize the results of the iPiano algorithm with different
settings and L-BFGS in Table \ref{table:mrfl2} and \ref{table:mrfl1}.
From these two tables, one can draw the common conclusion that iPiano
with a proper inertial term takes significantly less iterations
compared to the case without inertial term, and in practice $\beta
\approx 0.8$ is generally a good choice.

In Table \ref{table:mrfl2}, one can see that the iPiano algorithm with
$\beta = 0.8$ takes slightly more iterations and run time to reach a
solution of moderate accuracy (e.g., $\text{tol} = 10^{3}$) compared
with L-BFGS. However, for high accurate solutions (e.g., $\text{tol} =
10^{-5}$), this gap increases. For the case of the non-smooth
MRF-$\ell_1$ model, the result is just the reverse. It is shown in
Figure \ref{table:mrfl1}, that for reaching a moderately accurate
solution, iPiano with $\beta = 0.8$ consumes significantly less
iterations and run time, and for the solution of high accuracy, it
still can save much computation.

Figure\ref{fig:rates} plots the error $\mu_N$ over the number of
required iterations $N$ for both the MRF-$\ell_2$ and -$\ell_1$ models
using $\beta=0.8$. From the plots it becomes obvious that the
empirical performance of the iPiano algorithm is much better compared
to the worst-case convergence rate of $\mathcal{O}(1/N)$ as provided
in theorem~\ref{thm:conv-rate}.

The iPiano algorithm has an additional advantage of simplicity.  The
iPiano version without backtracking basically relies on matrix vector
products (filter operations in the denoising examples) and simple
pointwise operations. Therefore, the iPiano algorithm is well suited
for a parallel implementation on GPUs which an lead to speedup factors
of 20-30.

\subsubsection{Linear diffusion based image compression}
In this example we apply the iPiano algorithm to linear diffusion
based image compression. Recent works~\cite{GalicWWBBS08, SchmaltzWB09}
have shown that image compression based on linear and non-linear
diffusion can outperform the standard JPEG standard and even the more
advanced JPEG 2000 standard, when the interpolation points are
carefully chosen. Therefore, finding optimal data for interpolation is
a key problem in the context of PDE-based image compression.  There
exist only few prior works for this topic, see
e.g. \cite{MainbergerHWTJND11, HoeltgenSW13}, and the very recent
approach presented in \cite{HoeltgenSW13} defines the
state-of-the-art.

The problem of finding optimal data for homogeneous diffusion-based
interpolation is formulated as the following constrained minimization
problem:
\begin{align}\label{mask}
  \min_{\img, c} & \frac{1}{2}\|\img - \noisy\|_2^2 + \lambda\|c\|_1\\
  \text{s.t.}  \;  & C(\img - \noisy) - (\cI -C)L\img = 0\,, \nonumber
\end{align}
where $\noisy \in \R^N$ denotes the ground truth image, $\img \in
\R^N$ denotes the reconstructed image, and $c \in \R^N$ denotes the
inpainting mask, i.e. the characteristic function of the set of points
that are chosen for compressing the image. Furthermore, we denote by
$C = \diag(c) \in \R^{N \times N}$ the diagonal matrix with the vector
$c$ on its main diagonal, by $\cI$ the identity matrix and by $L\in
\R^{N \times N}$ the Laplacian operator. Compared to the original
formulation \cite{HoeltgenSW13}, we omit a very small quadratic term
$\frac{\varepsilon}{2}\|c\|_2^2$, because we find it unnecessary in
experiments.

Observe that if $c \in [0,1)^N$, we can multiply the constraint
equation in~\eqref{mask} from the left by $(\cI-C)^{-1}$ such
that it becomes
\[
E(c) (\img - \noisy) - L\img = 0\,,
\]
where $E(c) = \diag(c_1/(1-c_1), ..., c_N/(1-c_N))$. This shows that
problem~\eqref{mask} is in fact a reduced formulation of the bilevel
optimization problem
\begin{align}\label{eq:bilevel}
  \min_{c} & \frac{1}{2}\|\img(c) - \noisy\|_2^2 + \lambda\|c\|_1\\
  \text{s.t.}  & \quad \img(c) = \arg\min_\img \|D u\|_2^2 + \|E(c)^{\tfrac{1}{2}}(\img-\noisy)\|_2^2 \,, \nonumber
\end{align}
where $D$ is the nabla operator and hence $-L=D^\top D$.

Problem \eqref{mask} is non-convex due to the non-convexity of the
equality constraint. In~\cite{HoeltgenSW13}, the above problem is
solved by a successive primal-dual (SPD) algorithm, which successively
linearizes the non-convex constraint and solves the resulting convex
problem with the first-order primal-dual algorithm~\cite{CP11}. The
main drawback of SPD is, that it requires tens of thousands inner
iterations and thousands of outer iterations to reach a reasonable
solution. However, as we now demonstrate, iPiano can solve this
problem with higher accuracy in 1000 iterations.

Observe that we can rewrite the problem~\eqref{mask} by solving $\img$
from the constraints equation, which gives
\[
\img = A^{-1}C\noisy\,,
\]
where $A = C + (C - \cI)L$. In~\cite{MainbergerBWF11}, it is shown
that the $A$ is invertible as long as at least one element of $c$ is
non-zero, which is the case for non-degenerate problems. Substituting
back the above equation into~\eqref{mask}, we arrive at the following
optimization problem, which now only depends on the inpainting mask
$c$:
\begin{equation}\label{newmask}
  \min_{c}\frac 12\| A^{-1}C\noisy - \noisy\|_2^2 + \lambda\|c\|_1\,.
\end{equation}
Casting \eqref{newmask} in the form of \eqref{eq:problem-class}, we
have $f(c) = \frac 12\| A^{-1}C\noisy - \noisy\|_2^2$, and $g(c) =
\lambda\|c\|_1$.  In order to minimize the above problem using
iPiano, we need to calculate the gradient of $f$ with respect to $c$.
This is shown by the following lemma.
\begin{lemma}\label{cgradients}
  Let \[ f(c) = \frac 12\| A^{-1}C\noisy - \noisy\|_2^2\,,
\]
then
\begin{equation}\label{gradients}
\nabla f(c) = \diag(-(\cI + L)\img + \noisy)(A^\top)^{-1} (\img - \noisy)\,.
\end{equation}
\end{lemma}
\begin{proof}
Differentiating both sides of
\[
  f =\frac 12 \|\img - \noisy\|_2^2 = \frac 12 \scal{\img - \noisy}{\img - \noisy}\,,
\]
we obtain
\begin{equation}\label{differentiation}
  \mathrm{d}f = \scal{\mathrm{d}\img}{\img - \noisy}\,.
\end{equation}
In view of $\img = A^{-1}C\noisy$ and $\mathrm{d} A^{-1} =
-A^{-1}\mathrm{d}AA^{-1}$, we further have
\begin{eqnarray}
\mathrm{d}\img  &=& \mathrm{d}A^{-1}C\noisy + A^{-1}\mathrm{d}C\noisy \nonumber \\
   &=& -A^{-1}\mathrm{d}AA^{-1}C\noisy + A^{-1}\mathrm{d}C\noisy \nonumber \\
   &=& -A^{-1}\mathrm{d}A\img + A^{-1}\mathrm{d}C\noisy \nonumber\\
   &=&-A^{-1} \mathrm{d}C(\cI + L)\img + A^{-1}\mathrm{d}C\noisy\nonumber \\
   &=& A^{-1}\mathrm{d}C(-\left (\cI + L)\img + \noisy\right)\,. \nonumber
\end{eqnarray}
Let $t = -(\cI + L)\img + \noisy \in \R^N$, and since $C$ is a
diagonal matrix, we have
\[
  \mathrm{d}Ct = \diag(\mathrm{d}c) t = \diag(t)\mathrm{d}c\,,
\]
and hence
\begin{equation}\label{du}
  \mathrm{d}\img = A^{-1}\diag(t)\mathrm{d}c\,.
\end{equation}
By substituting \eqref{du} into \eqref{differentiation}, we obtain
\begin{eqnarray}
\mathrm{d}f &=& \scal{(A^{-1}\diag(t))\mathrm{d}c}{\img - \noisy} \nonumber \\
  &=& \scal{\mathrm{d}c}{(A^{-1}\diag(t))^\top (\img - \noisy)}\,.\nonumber
\end{eqnarray}
Finally, the gradient is given by
\begin{eqnarray}
\nabla f  &=& (A^{-1}\diag(t))^\top (\img - \noisy)  \\
  &=& \diag(-(\cI + L)\img + \noisy)(A^\top)^{-1} (\img - \noisy) \nonumber\,.
\end{eqnarray}
\qquad\end{proof}
\bigskip

Finally, we need to compute the proximal map with respect to $g(c)$
which is again given by a pointwise application of the shrinkage
operator~\eqref{softshrinkage}.

Now, we can make use of the iPiano algorithm to solve the problem
\eqref{newmask}. We set $\beta = 0.8$, which generally performs very
well in practice. We additionally accelerate the SPD algorithm used in
the previous work~\cite{HoeltgenSW13} by applying the diagonal
preconditioning technique~\cite{PC11}, which significantly reduces the
required iterations for the primal-dual algorithm in the inner loop.

Figure \ref{fig:compression} shows examples of finding optimal
interpolation data for the three test images.  Table \ref{truiresult}
summarizes the results of two different algorithms. Regarding the
reconstruction quality, we make use of the mean squared error (MSE) as
an error measurement to keep consistent with previous work, which is
computed by
\[
MSE(\img,\noisy) = \frac 1 N \sum_{i=1}^{N}(\img_i - \noisy_i)^2\,.
\]
From Table \ref{truiresult}, one can see that the Successive PD
algorithm requires $200 \times 4000$ iterations to converge. iPiano
only needs 1000 iterations to reach already a lower energy. Note that
in each iteration of the iPiano algorithm, two linear systems have to
be solved. In our implementation we use the Matlab ``backslash''
operator which effectively exploits the strong sparseness of the
systems. A lower energy basically implies that iPiano can solve the
minimization problem \eqref{mask} better. Regarding the final
compression result, usually the result of iPiano has slightly less
density, but slightly worse MSE. Following the
work~\cite{MainbergerHWTJND11}, we also consider the so-called gray
value optimization (GVO) as a post-processing step to further improve
the MSE of the reconstructed images.

\begin{figure}[t!]
\begin{center}
  {\includegraphics[width=0.33\linewidth]{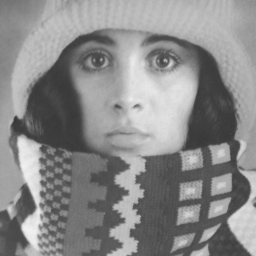}}\hfill
  {\includegraphics[width=0.33\linewidth]{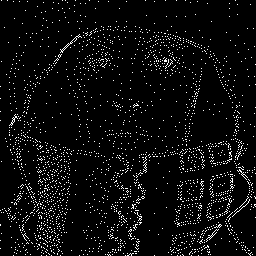}}\hfill
  {\includegraphics[width=0.33\linewidth]{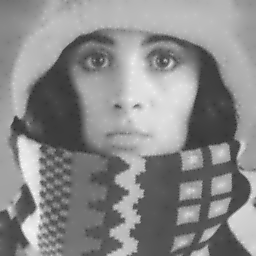}}\\
\vspace*{0.2cm}
  {\includegraphics[width=0.33\linewidth]{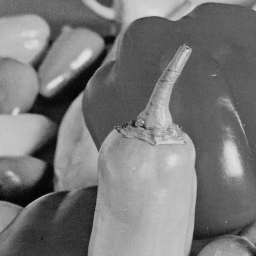}}\hfill
  {\includegraphics[width=0.33\linewidth]{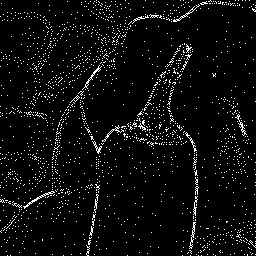}}\hfill
  {\includegraphics[width=0.33\linewidth]{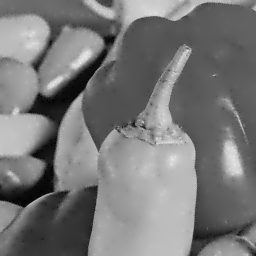}}\\
  \subfigure[Test image ($256 \times 256$)]{\includegraphics[width=0.33\linewidth]{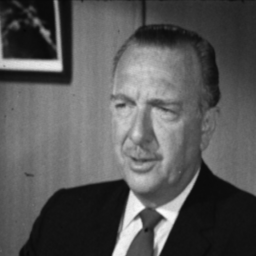}}\hfill
  \subfigure[Optimized mask]{\includegraphics[width=0.33\linewidth]{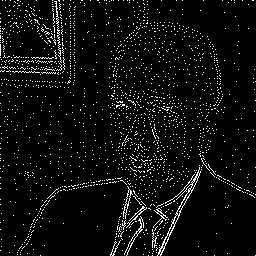}}\hfill
  \subfigure[Reconstruction]{\includegraphics[width=0.33\linewidth]{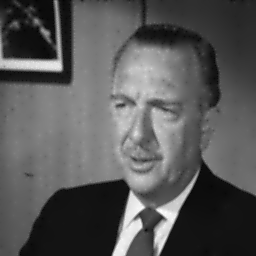}}\\
\end{center}
\caption{Examples of finding optimal inpainting mask for Laplace interpolation based image compression by using iPiano.
\textbf{First row:} Test image \textit{trui} of size $256 \times 256$.
Parameter $\lambda = 0.0036$, the optimized mask has a density of 4.98\% and the MSE of the reconstructed image is 16.89.
\textbf{Second row:} Test image \textit{peppers} of size $256 \times 256$.
Parameter $\lambda = 0.0034$, the optimized mask has a density of 4.84\% and the MSE of the reconstructed image is 18.99.
\textbf{Third row:} Test image \textit{walter} of size $256 \times 256$.
Parameter $\lambda = 0.0018$, the optimized mask has a density of 4.82\% and the MSE of the reconstructed image is 8.03.}
\label{fig:compression}
\end{figure}

\begin{table}[ht]
\centering
\begin{tabular}{p{1.2cm}p{1.5cm}p{1.5cm}p{1.5cm}p{1cm}p{1cm}p{1.5cm}}
  \toprule
    Test image & Algorithm & Iterations & Energy & Density & MSE & with GVO\\
\midrule
\midrule
    \multirow{2}{*}{trui} &  iPiano & 1000 & 21.574011 & 4.98\% & 17.31 &16.89\\
    \cmidrule{2-7} & SPD & 200/4000 & 21.630280 & 5.08\% & 17.06 &16.54 \\
  \midrule
\multirow{2}{*}{peppers} &  iPiano & 1000 & 20.631985 & 4.84\% & 19.50 &18.99\\
    \cmidrule{2-7} & SPD & 200/4000 & 20.758777 & 4.93\% & 19.48 &18.71 \\
\midrule
\multirow{2}{*}{walter} &  iPiano & 1000 & 10.246041 & 4.82\% & 8.29 &8.03\\
    \cmidrule{2-7} & SPD & 200/4000 & 10.278874 & 4.93\% & 8.01 &7.72 \\
  \bottomrule
\end{tabular}
\caption{Summary of two algorithms for three test images.}\label{truiresult}
\end{table}


\section{Conclusions}
In this paper, we have proposed a new optimization algorithm, which we call \-{iPiano}. It is applicable to a broad class of non-convex problems. More specifically, it addresses objective functions, which are composed as a sum of a differentiable (possibly non-convex) and a convex (possibly non-differentiable) function. The basic methodologies have been derived from the forward-backward splitting algorithm and the Heavy-ball method.

Our theoretical convergence analysis is divided into two steps. First, we have proved an abstract convergence result about inexact descent methods. Then, we analyze the convergence of iPiano. For iPiano, we have proved that the sequence of function values converges, that the subsequence of arguments generated by the algorithm is bounded, and that every limit point is a critical point of the problem. Requiring the \KL property for the objective function establishes deeper insights into the convergence behavior of the algorithm. Using the abstract convergence result, we have shown that the whole sequence converges and the unique limit point is a stationary point.

The analysis includes an examination of the convergence rate. A rough upper bound of $O(1/n)$ has been found for the squared proximal residual. Experimentally, iPiano has been shown to have a much faster convergence rate.

Finally, the applicability of the algorithm has been demonstrated and iPiano achieved state-of-the-art performance. The experiments comprised denoising and image compression. In the first two experiments, iPiano helped learning a good prior for the problem. In the case of image compression, iPiano has demonstrated its use in a huge optimization problem for computing an optimal mask for a Laplacian PDE-based image compression method.

In summary, iPiano has many favorable theoretical properties, is simple and efficient.  Hence, we recommend it as a standard solver for the considered class of problems.

\section{Acknowledgements}

We are grateful to Joachim Weickert for discussions about the image compression by diffusion problem.

{\small
\bibliographystyle{siam}
\bibliography{ochs}
}

\end{document}